\documentclass[11pt]{article}

\usepackage{etoolbox}
\newtoggle{colt}
\togglefalse{colt}

\iftoggle{colt}{}{
\usepackage[sort&compress,numbers]{natbib}
}

\usepackage{boxedminipage}
\usepackage{multirow,nicefrac}
\usepackage{makecell,upgreek}
\usepackage{footnote}
\usepackage{longtable}
\usepackage{tablefootnote}
\usepackage[T1]{fontenc}
\usepackage{verbatim} 
\usepackage[utf8]{inputenc}

\usepackage{booktabs}
\usepackage{tabularx}
\usepackage{adjustbox}

\iftoggle{colt}{}{
\usepackage[margin=1in]{geometry}
}
\usepackage{algorithm}
\usepackage{algpseudocode}

\usepackage[dvipsnames]{xcolor}

\usepackage{graphicx}
\usepackage{caption}
\usepackage{subcaption}

\usepackage{colortbl}
\definecolor{lightgray}{gray}{0.9}  %

\usepackage{amsmath}
\usepackage{amsthm}
\iftoggle{colt}{
  
}{
\usepackage[breaklinks=true]{hyperref}
}
\hypersetup{
	colorlinks=true,
	linkcolor=blue,
	citecolor=blue,
	urlcolor=blue
}

\usepackage[capitalise,nameinlink]{cleveref}

\usepackage{amsfonts,amssymb,mathrsfs}
\usepackage{mathtools}

\usepackage{appendix}

\theoremstyle{plain}
	\newtheorem{theorem}{Theorem}

	\newtheorem{lemma}{Lemma}
	
	\newtheorem{corollary}{Corollary}
	\newtheorem{proposition}{Proposition}

	\theoremstyle{definition}

	\newtheorem{definition}{Definition}

  \newtheorem{assumption}{Assumption}[section]

\Crefname{appendix}{Appendix}{Appendices}

\usepackage{autonum}

\usepackage[scaled]{beramono}
\usepackage{listings}
\definecolor{varorange}{RGB}{230,126,34}     %
\definecolor{opblue}{RGB}{52,152,219}        %
\definecolor{ifred}{RGB}{192,57,43}          %
\definecolor{commentgray}{RGB}{150,150,150}  %
\definecolor{codebg}{rgb}{0.98,0.98,0.98}    %

\newsavebox\codebox

\usepackage{authblk}

\newcommand{\norm}[1]{\left\|#1\right\|}
\newcommand{\normop}[1]{\left\|#1\right\|_{\mathrm{op}}}

\newcommand{\inprod}[2]{\left\langle #1, #2 \right\rangle}

\newcommand{\rr}{\mathbb{R}}
\newcommand{\ee}{\mathbb{E}}

\newcommand{\eye}{\mathbf{I}}

\DeclareMathOperator*{\argmin}{argmin}

\DeclareMathOperator{\Cov}{Cov}
\DeclareMathOperator{\trace}{Tr}

\newcommand{\mustar}{\mu^{\star}}
\newcommand{\muhat}{\widehat{\mu}}
\newcommand{\ustar}{u^{\star}}

\DeclareMathOperator{\diag}{diag}
\newcommand{\thetema}[1][t]{\widehat{\theta}_{\mathsf{EMA}, #1}}
\newcommand{\ema}{\textsf{EMA}}
\newcommand{\bema}{\textsf{BEMA}}
\newcommand{\bSigma}{\mathbf{\Sigma}}
\newcommand{\thetahat}{\widehat{\theta}}
\newcommand{\bPdot}{\dot{\mathbf{P}}}

\newcommand{\thetabar}{\overline{\theta}}
\newcommand{\cmin}{c_{\mathrm{min}}}
\newcommand{\thetasgd}{\thetahat^{\mathrm{SGD}}}
\newcommand{\pace}{\textsf{PACE}}

\renewcommand{\epsilon}{\varepsilon}

\def\ddefloop#1{\ifx\ddefloop#1\else\ddef{#1}\expandafter\ddefloop\fi}
\def\ddef#1{\expandafter\def\csname bb#1\endcsname{\ensuremath{\mathbb{#1}}}}
\ddefloop ABCDEFGHIJKLMNOPQRSTUVWXYZ\ddefloop

\def\ddefloop#1{\ifx\ddefloop#1\else\ddef{#1}\expandafter\ddefloop\fi}
\def\ddef#1{\expandafter\def\csname frak#1\endcsname{\ensuremath{\mathfrak{#1}}}}
\ddefloop ABCDEFGHIJKLMNOPQRSTUVWXYZ\ddefloop

\def\ddefloop#1{\ifx\ddefloop#1\else\ddef{#1}\expandafter\ddefloop\fi}
\def\ddef#1{\expandafter\def\csname fr#1\endcsname{\ensuremath{\mathfrak{#1}}}}
\ddefloop ABCDEFGHIJKLMNOPQRSTUVWXYZ\ddefloop

\def\ddefloop#1{\ifx\ddefloop#1\else\ddef{#1}\expandafter\ddefloop\fi}
\def\ddef#1{\expandafter\def\csname eul#1\endcsname{\ensuremath{\EuScript{#1}}}}
\ddefloop ABCDEFGHIJKLMNOPQRSTUVWXYZ\ddefloop

\def\ddefloop#1{\ifx\ddefloop#1\else\ddef{#1}\expandafter\ddefloop\fi}
\def\ddef#1{\expandafter\def\csname scr#1\endcsname{\ensuremath{\mathscr{#1}}}}
\ddefloop ABCDEFGHIJKLMNOPQRSTUVWXYZ\ddefloop

\def\ddefloop#1{\ifx\ddefloop#1\else\ddef{#1}\expandafter\ddefloop\fi}
\def\ddef#1{\expandafter\def\csname b#1\endcsname{\ensuremath{\mathbf{#1}}}}
\ddefloop ABCDEFGHIJKLMNOPQRSTUVWXYZ\ddefloop

\def\ddefloop#1{\ifx\ddefloop#1\else\ddef{#1}\expandafter\ddefloop\fi}
\def\ddef#1{\expandafter\def\csname bhat#1\endcsname{\ensuremath{\hat{\mathbf{#1}}}}}
\ddefloop ABCDEFGHIJKLMNOPQRSTUVWXYZ\ddefloop

\def\ddefloop#1{\ifx\ddefloop#1\else\ddef{#1}\expandafter\ddefloop\fi}
\def\ddef#1{\expandafter\def\csname btil#1\endcsname{\ensuremath{\tilde{\mathbf{#1}}}}}
\ddefloop ABCDEFGHIJKLMNOPQRSTUVWXYZ\ddefloop

\def\ddefloop#1{\ifx\ddefloop#1\else\ddef{#1}\expandafter\ddefloop\fi}
\def\ddef#1{\expandafter\def\csname bst#1\endcsname{\ensuremath{\mathbf{#1}^\star}}}
\ddefloop ABCDEFGHIJKLMNOPQRSTUVWXYZ\ddefloop

\def\ddefloop#1{\ifx\ddefloop#1\else\ddef{#1}\expandafter\ddefloop\fi}
\def\ddef#1{\expandafter\def\csname bst#1\endcsname{\ensuremath{\mathbf{#1}^\star}}}
\ddefloop abcdeghijklmnopqrstuvwxyz\ddefloop

\def\ddefloop#1{\ifx\ddefloop#1\else\ddef{#1}\expandafter\ddefloop\fi}
\def\ddef#1{\expandafter\def\csname bhat#1\endcsname{\ensuremath{\hat{\mathbf{#1}}}}}
\ddefloop abcdefghijklmnopqrstuvwxyz\ddefloop

\def\ddefloop#1{\ifx\ddefloop#1\else\ddef{#1}\expandafter\ddefloop\fi}
\def\ddef#1{\expandafter\def\csname b#1\endcsname{\ensuremath{\mathbf{#1}}}}
\ddefloop abcdeghijklnopqrstuvwxyz\ddefloop

\def\ddefloop#1{\ifx\ddefloop#1\else\ddef{#1}\expandafter\ddefloop\fi}
\def\ddef#1{\expandafter\def\csname barb#1\endcsname{\ensuremath{\bar{\mathbf{#1}}}}}
\ddefloop abcdefghijklmnopqrstuvwxyz\ddefloop

\def\ddef#1{\expandafter\def\csname c#1\endcsname{\ensuremath{\mathcal{#1}}}}
\ddefloop ABCDEFGHIJKLMNOPQRSTUVWXYZ\ddefloop
\def\ddef#1{\expandafter\def\csname h#1\endcsname{\ensuremath{\widehat{#1}}}}
\ddefloop ABCDEFGHIJKLMNOPQRSTUVWXYZ\ddefloop
\def\ddef#1{\expandafter\def\csname hc#1\endcsname{\ensuremath{\widehat{\mathcal{#1}}}}}
\ddefloop ABCDEFGHIJKLMNOPQRSTUVWXYZ\ddefloop
\def\ddef#1{\expandafter\def\csname t#1\endcsname{\ensuremath{\widetilde{#1}}}}
\ddefloop ABCDEFGHIJKLMNOPQRSTUVWXYZ\ddefloop
\def\ddef#1{\expandafter\def\csname tc#1\endcsname{\ensuremath{\widetilde{\mathcal{#1}}}}}
\ddefloop ABCDEFGHIJKLMNOPQRSTUVWXYZ\ddefloop

\usepackage[suppress]{color-edits}
\addauthor{ab}{red}
\addauthor{ka}{green}

\Crefname{assumption}{Assumption}{Assumptions}

\title{Training for the Model You Return: Improving Optimization for Iterate-Averaged Language Models\thanks{The code used to run our evaluations can be found at \url{https://github.com/leonou2010/pace-optimizer}.}}

\usepackage[section]{placeins}

\author[1]{Kwok Chun Au}
\author[1,2]{Adam Block}

\affil[1]{Department of Computer Science, Columbia University}
\affil[2]{Department of Electrical Engineering, Columbia University}

\date{}

\begin{document}

\maketitle

\begin{abstract}
Many modern Language Model (LM) pipelines return an averaged model, such as an exponential moving average of the training iterates, rather than the final iterate itself. This raises a fundamental question: given that we will return an iterate average, how should we change training to improve the performance of this average? We study this question by formulating optimizer design for the iterate-average estimator as an optimal-control problem. In a continuous-time stochastic quadratic model, we solve for the control strategy that minimizes the error of the returned average subject to a penalty on the size of the intervention. A practical approximation to this controller yields \pace, a lightweight wrapper around AdamW that pulls the live weights toward their exponential moving average with a clipped, per-coordinate control strength. We prove that a stylized version of \pace\ converges at the standard stochastic convex optimization rate, up to a factor depending on the averaging rule, while in the quadratic setting it can strictly improve the limiting squared error of the iterate-average estimator and can do so by an arbitrarily large factor on some instances. Empirically, our results suggest that \pace\ improves over AdamW and EMA-evaluated AdamW in supervised fine-tuning of 1-2B parameter LMs and in GPT-2 pretraining on FineWeb for a wide range of learning rates, decay schedules, and other hyperparameters.

\end{abstract}

\section{Introduction}

Inspired by the classical theory of stochastic convex optimization \citep{ruppert1988efficient,polyak1992acceleration}, many Language Model (LM) pipelines use some form of model averaging, such as an exponential moving average (\ema) of training iterates, to stabilize optimization and improve the performance of the returned model \citep{izmailov2018averaging,kaddour2022stop,block2025ema,block2024butterfly,busbridge2023scale,olmo20242,lambert2024tulu,wortsman2022model,smollm2}.  A vast body of work has been devoted to understanding LM training, but the ubiquity of iterate averaging in practice raises a fundamental question: \emph{given that an averaged LM will be returned, how should we change training to improve the performance of this average?}

While prior work has implicitly addressed this question by observing that iterate averaging can allow for more aggressive training through higher learning rates or more limited learning rate decay \citep{defazio2024schedule,block2024butterfly,block2025ema}, in this work, we take a direct approach by asking whether we can design a simple modification to the base optimization algorithm that explicitly controls the training trajectory to improve the performance of the iterate-average estimator.  Our starting point is to follow \citet{block2025ema} and observe that once we commit to returning an average of the training trajectory, the optimization algorithm is no longer just producing a final iterate, but rather a statistic of the entire trajectory: we may thus rephrase our key question by asking how we can best modify the training trajectory so as to make the returned model average maximally close to the optimum.

To answer this question, we take a principled approach to algorithm design by deriving a new optimization algorithm from an optimal-control formulation of a stochastic quadratic optimization problem.  Indeed, while LM training is obviously not quadratic, many recent works have derived practically effective optimization interventions by designing algorithms in analytically tractable settings such as convex, quadratic, or linear problems \citep{vyas2024soap,zhang2019lookahead,gupta2018shampoo,block2025ema,duchi2011adaptive}.  In particular, our derivation begins in a deliberately simplified setting: continuous-time stochastic gradient descent on a quadratic objective with additive noise, where the `modification' is realized as an additive \emph{control}.  
In this model, the uncontrolled dynamics reduce to the familiar Ornstein--Uhlenbeck approximation of noisy gradient descent near a quadratic minimum \citep{mandt2015continuous,block2025ema,kutoyants2004statistical}.  We then add a control input to the dynamics and choose this input to minimize the squared error of the returned iterate average, subject to a quadratic penalty on the size of the intervention.  This gives a linear-quadratic stochastic control problem \citep{georgiou2013separation,yong1999stochastic,kwakernaak1972linear}.  Solving this problem yields an optimal controller with a natural interpretation: the algorithm should pull the live iterate toward an estimate of the optimum and toward consistency with the accumulated average of the trajectory.

\begin{figure}[t]
    \centering
    \includegraphics[width=\textwidth]{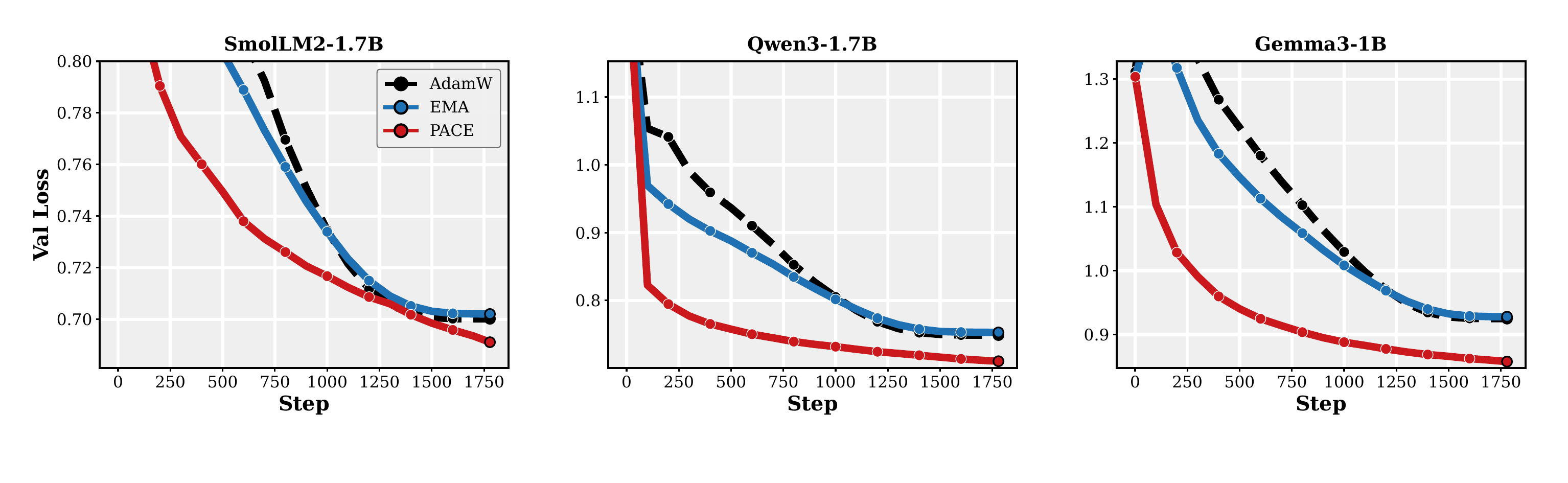}
    \caption{
        \textbf{Performance of \pace\ on fine-tuning.} Validation cross-entropy on \texttt{smol-smoltalk} for SmolLM2-1.7B (\textbf{left}), Qwen3-1.7B (\textbf{middle}), and Gemma3-1B (\textbf{right}).  For each model, \pace\ uses a constant learning rate, while the AdamW and \ema\ baselines use their best learning-rate-decay schedule (cosine or WSD).  \pace\ strictly improves on both baselines on all three models; see \Cref{fig:appendix_multiseed_robust} for a three-seed robustness study.}
    \label{fig:train_curves}
\end{figure}

The exact optimal controller is not directly implementable in modern LM training because it depends on quantities unknown to the learner.  We therefore turn the controller into a practical algorithm through a sequence of heuristic approximations.  We call the resulting method \pace\ (\Cref{alg:controlled-bema}), a lightweight wrapper around AdamW \citep{loshchilov2017fixing} that pays an additional memory cost of model weights in order to better stabilize the training process.  Because \pace\ is derived in such a simplified setting, it is not a priori obvious whether the resulting algorithm has any concrete theoretical guarantees.  We show that a stylized version of the algorithm has convergence guarantees for general convex losses with unbiased stochastic gradients:

\begin{theorem}[Informal version of \cref{thm:convex_convergence}]\label{thm:convex_convergence_informal}
    For any convex loss with unbiased stochastic gradients, the iterate-average estimator returned by \pace\ converges to the minimum at the standard stochastic convex optimization rate, up to a factor depending on the strength of the \ema.
\end{theorem}
In essence, \Cref{thm:convex_convergence_informal} shows that our proposed method is never much worse than SGD for convex objectives, even when the quadratic assumption used in the derivation is dropped.  We then show that, in the quadratic setting from which the algorithm was derived, the controlled dynamics can strictly improve the limiting squared error of the iterate-average estimator relative to the uncontrolled dynamics, and that this improvement can be arbitrarily large on certain quadratic instances:
\begin{proposition}[Informal version of \cref{prop:improved_convergence_quadratic,prop:arbitrarily_large_improvement}]\label{prop:quadratic_improvement_informal}
    In the quadratic setting from which \pace\ is derived, the optimal control strategy can strictly improve the limiting squared error of the iterate-average estimator relative to the uncontrolled dynamics, and this improvement can be arbitrarily large on certain quadratic instances.
\end{proposition}
\iftoggle{colt}{}{
To summarize, in theory, \pace\ is never that much worse than SGD for convex objectives, but can be much better when the objective is locally quadratic.
}

We complement our theory with experiments on LM fine-tuning and pretraining, results of the former of which can be seen in \Cref{fig:train_curves}.  In particular, across three different models of size between 1B and 2B parameters, we find that \pace\ strictly improves over both AdamW and an \ema\ thereof.  In \Cref{sec:experiments}, we present a more detailed analysis of the performance of \pace\ across a range of hyperparameters as well as explore the efficacy of \pace\ on a small pretraining setup at the Chinchilla-optimal token budget \citep{hoffmann2022chinchilla}.  We find that \pace\ is competitive with the Schedule-Free algorithm of \citet{defazio2024schedule}, which also uses iterate averaging to stabilize training and remove the need for learning rate decay and that \pace\ remains competitive with AdamW even when the latter is allowed to decay the learning rate to zero, suggesting that \pace\ can stabilize training in a regime that is not subject to learning rate decay.  \iftoggle{colt}{We now briefly summarize some related works, before moving on to the formal problem setup, algorithm derivation, and rigorous guarantees in \Cref{sec:theory}.}{

\paragraph{Related work.} We now briefly discuss some of the key related works, before moving on to the formal problem setup, algorithm derivation, and rigorous convergence guarantees in \Cref{sec:theory}.
}

\paragraph{Iterate averaging in deep learning.} Originally introduced from the theory of stochastic convex optimization \citep{ruppert1988efficient,polyak1992acceleration}, iterate averaging has been shown to be a powerful tool for stabilizing optimization and improving generalization in deep learning \citep{izmailov2018averaging,kaddour2022stop,block2025ema,block2024butterfly,busbridge2023scale}.  Many approaches have been proposed, including stochastic weight averaging \citep{izmailov2018averaging}, exponential moving averages, and even more complex schemes like model soups \citep{wortsman2022model} and virtually all modern open-source LMs used some form of iterate averaging \citep{smollm2,olmo20242,lambert2024tulu}.  As in our work, there has been much literature on using continuous-time analysis to understand the behavior of iterate averaging in deep learning \citep{busbridge2023scale,mandt2015continuous,malladi2022sdes,block2024butterfly}, with \citet{block2025ema} being the most directly relevant to the present paper.  Unlike that work, which uses continuous-time analysis in a noisy quadratic model to propose a new averaging scheme, here we commit to the standard exponential moving average and instead use continuous-time analysis to design a new optimization algorithm that explicitly controls the training trajectory to improve the performance of this standard averaging scheme.

\paragraph{Comparison to similar optimizers.} Several authors have proposed analyzing optimization through the lens of control theory \citep{lessard2016analysis,davtyan2022koala,xia2025koala}, although the former primarily investigates convergence bounds of existing optimizers and the latter two use the Kalman filter to derive an optimization intervention that, while interesting, does not appear to scale to LMs of modern size.  With regard to the final algorithm, several prior optimizers have similar update rules.  \citet{zhang2015deep,pagliardini2024ademamix} also consider averaging during training (instead of in an offline, post-hoc manner).  The former applies averaging to gradients as opposed to iterates and can be thought of as a more sophisticated version of momentum, whereas the latter has a very different update rule from \pace\ and is intended to solve communication bottlenecks in distributed training.
The authors of \citet{zhang2019lookahead} propose a method that takes gradient steps and periodically transports the live weights toward a slow-moving EMA; our approach recovers a similar update when the pullback value is large, although empirically we find that \pace\ improves performance with intermediate pullback values, improving on that prior work.  Of greatest relevance is the schedule-free (SF) update proposed by \citet{defazio2024schedule}, which also uses iterate averaging to stabilize training and remove the need for learning rate decay, although that work proposes a different algorithm.  Our method is competitive with SF empirically, although pays for improved theoretical convergence in the quadratic setting with a slower theoretical guarantee in the general convex setting. Moreover, our na{\"i}ve implementation of \pace\ uses an additional memory cost of model weights, as we incorporate momentum.

\paragraph{Notation.} We use lowercase letters to denote scalars and uppercase bold letters to denote matrices.  We always denote the standard Brownian motion in $\rr^d$ by $W_t$ and $\norm{\cdot}$ refers to the euclidean norm.  We use $\diag(\alpha_1, \dots, \alpha_d)$ to denote the diagonal matrix with entries $\alpha_1, \dots, \alpha_d$ on the diagonal.  For vector functions depending on time, we use a subscript to denote the time dependence (e.g. $\theta_t$), unless we are concerned with a single coordinate, in which case we use parentheses (e.g. $\theta_i(t)$).

\section{Algorithm Derivation and Guarantees}\label{sec:theory}
\iftoggle{colt}{}
{
In this section, we formally introduce the problem in which we are interested as well as derive a rigorous solution in a simplified setting.  We then demonstrate that a simple, heuristic approximation to this solution has convergence guarantees for convex problems and provably outperforms standard SGD in the quadratic setting.
}

\subsection{Formal Problem Setup}\label{ssec:problem_setup}
Formally, we are interested in \emph{stochastic optimization}, where for some loss function $F: \rr^d \to \rr$ with minimum $\mustar = \argmin_{\theta} F(\theta)$, we wish to return some $\thetahat$ that is close to $\mustar$.  As is standard in the stochastic optimization literature \citep{polyak1992acceleration,ruppert1988efficient,defazio2024schedule,block2024butterfly,loshchilov2017fixing}, we will assume only that we have noisy gradient access to $F$, i.e., for any $\theta \in \rr^d$, we can sample some $g$ such that $\ee[g] = \nabla F(\theta)$.

Our starting point is to recall the following empirical observation about the behavior of language model training: iterate averaging schemes like \ema\ and \bema\ are extraordinarily effective at improving the performance of LMs, both through stabilizing the optimization and through improving generalization \citep{izmailov2018averaging,kaddour2022stop,olmo20242}. While there has been a large body of work dedicated to understanding and improving such iterate averaging schemes, we instead focus on the question of how to improve training algorithms under the assumption that we will be returning an iterate average.  In particular, we ask the following fundamental question:
\begin{quotation}
    \emph{Given that we will return an iterate average after training, how should we modify optimization algorithms to minimize the loss of the returned weights?}
\end{quotation} 
We will answer this question theoretically in a simplified setting, and empirically by training large language models with modern optimization algorithms.  In theory, following a long line of work in optimization for deep learning, we will derive an answer to this question by considering a quadratic model \citep{duchi2011adaptive,gupta2018shampoo,zhang2019lookahead,block2025ema,block2024butterfly,busbridge2023scale}, and then prove convergence guarantees for our algorithm in a more general convex setting.  In particular, for our derivation, we will make the following four assumptions. First, the loss function $F$ is a quadratic function of the form \iftoggle{colt}{$F(\theta) = \nicefrac 12 \left( \theta - \mustar \right)^\top \bA \left( \theta - \mustar \right)$}{
    \begin{align}
        F(\theta) = \frac{1}{2} \left( \theta - \mustar \right)^\top \bA \left( \theta - \mustar \right),
    \end{align}
}
for a \emph{diagonal} positive definite matrix $\bA$ (for the general, non-diagonal case, see \Cref{app:stoch_control_derivation}).  Second, in order to further simplify the analysis, we will follow prior work \citep{block2025ema,block2024butterfly,malladi2022sdes,mandt2015continuous,block2020generative} and consider a continuous time limit of the optimization process, where the dynamics of the parameters $\theta_t$ are given by a stochastic differential equation (SDE).  Third, we will assume that we are returning a simple time average of the iterates, as opposed to more sophisticated schemes like \ema\ and \bema. Fourth, we will consider a Bayesian setting, where $\mustar$ is a random variable with a known prior distribution.  Finally, we will focus on the case of vanilla SGD with stationary, homoscedastic noise, i.e. the noise covariance $\bSigma$ is constant and does not depend on $\theta_t$.  We will model the modification to the base optimization algorithm as a \emph{control} input $u_t$ that can be added to the dynamics of $\theta_t$, affecting the trajectory of the optimization process.  Combining these assumptions, we thus return $\thetema[T]$:
\begin{align}\label{eq:sde}
    \thetema[T] = \frac 1T \int_0^T \theta_t d t \quad \text{where} \quad d \theta_t = \left[\bA \left( \mustar - \theta_t \right) + u_t\right] d t + \bSigma^{1/2} d W_t,
\end{align}
and $W_t$ a standard Brownian motion in $\rr^d$ \citep{kutoyants2004statistical,yong1999stochastic}.  Critically, we wish $\thetema[T]$ to be a good estimator of $\mustar$ while exerting as little control effort as possible; one way to formalize this desideratum is to seek a control strategy $u_t$ that minimizes the following objective:
\iftoggle{colt}{
\begin{align}\label{eq:objective}
    J_T(u) = \ee\left[ \|\thetema[T] - \mustar\|^2 + \lambda \cdot \int_0^T \norm{u_t}^2 d t \right],
\end{align}
}{
\begin{align}\label{eq:objective}
    J_T(u) = \ee\left[ \norm{\thetema[T] - \mustar}^2 + \lambda \cdot \int_0^T \norm{u_t}^2 d t \right],
\end{align}
}
where $\lambda > 0$ is a regularization parameter that controls the tradeoff between the quality of the returned estimator $\thetema[T]$ and how far from the base optimization algorithm the control strategy $u_t$ is allowed to deviate.  We will refer to the problem of minimizing $J_T(u)$ over all progressively measurable control strategies $u_t$ as the \emph{optimal control problem} for the iterate-average estimator.  Note that in the case that $u_t = 0$ for all $t$, \eqref{eq:sde} reduces to the well-known Ornstein-Uhlenbeck process \citep{kutoyants2004statistical,mandt2015continuous}, and we recover the setting considered in \citet{block2025ema}.

\subsection{Derivation of the Optimal Control Strategy}\label{ssec:derivation}
We now present the optimal control solution for the iterate-average estimator.
\begin{theorem}\label{thm:optimal_control_solution}
    Suppose that $\mustar$ is drawn from a known prior distribution and $\theta_t$ evolves according to \eqref{eq:sde} for some control strategy $u_t$\footnote{Formally, we assume $u_t$ is \emph{progressively measurable}, i.e. can only include information from observations up to time $t$.} and that $\bA = \diag(\alpha_1, \dots, \alpha_d)$ is a diagonal positive definite matrix.  Let $\muhat_t = \ee\left[ \mustar | \theta_s, s \leq t \right]$.  Then the optimal control strategy $u_t$ that minimizes the objective $J_T(u)$ defined in \eqref{eq:objective} is given for each coordinate $i \in \{1, \dots, d\}$ by
    \begin{align}\label{eq:optimal_control_solution}
        \ustar_i(t) &= \frac{\left( 1 - e^{- \alpha_i (T - t)} \right)^2}{\alpha_i^2 \lambda T^2 + T - t - \nicefrac{2}{ \alpha_i} (1 - e^{- \alpha_i (T - t)}) + \frac{1 - e^{- 2 \alpha_i (T - t)}}{2 \alpha_i}} (\muhat_i(t) - \theta_i(t)) \\
        &\quad + \frac{\alpha_i(1 - e^{- \alpha_i (T - t)})}{\alpha_i^2 \lambda T^2 + T - t - \nicefrac{2}{ \alpha_i} (1 - e^{- \alpha_i (T - t)}) + \frac{1 - e^{- 2 \alpha_i (T - t)}}{2 \alpha_i}} \left( t \cdot \muhat_i(t) - \int_0^t \theta_i(s) d s \right).
    \end{align}
\end{theorem}
We defer a complete proof of this result to \Cref{app:stoch_control_derivation}, as well as generalizing beyond the diagonal $\bA$ setting.  After some algebra and augmenting the state, the analysis rests on classical results from the theory of linear quadratic stochastic control \citep{georgiou2013separation,kwakernaak1972linear} including solving the Riccati equation and applying the principle of certainty equivalence.  While \eqref{eq:optimal_control_solution} is attractive in that it is a closed-form expression for the optimal control strategy, it is not easily implementable in practice, as it depends on $\muhat_t$, the posterior mean of $\mustar$ given the trajectory of $\theta_s$ up to time $t$.  Even in the case of a Gaussian prior, which would imply that $\muhat_t$ is given by a Kalman filter \citep{kutoyants2004statistical,georgiou2013separation,yong1999stochastic}, the resulting control strategy would be difficult to implement in practice, especially at the scale of language models.  Thus, we simplify \eqref{eq:optimal_control_solution} to obtain a more practical control strategy by making the following approximations: (i) we replace $\muhat_t \approx \thetema[t]$ and (ii) we suppose that $t \ll T$.  The first assumption is connected with the fundamental motivation of our work: we wish to return $\thetema[T]$ as an estimate of $\mustar$ at the end of training, so it is natural to use $\thetema[t]$ as an estimate of $\mustar$ at time $t$.  The second assumption is motivated primarily by a desire for simplicity and the observation that the optimal control strategy is most impactful in the early stages of training, when the optimization trajectory is far from the optimum and thus there is more room for improvement.   Assuming $T \gg1$, we approximate $T -t \approx T$ and $\exp(-\alpha_i (T - t)) \approx 0$. With these approximations, we obtain the following practical control:
\begin{align}\label{eq:practical_control_strategy}
    u_t = T^{-1} \left( \eye + \lambda \bA^2 T\right)^{-1} \left( \thetema[t] - \theta_t\right).
\end{align}

While \eqref{eq:practical_control_strategy} is a valid approximation to the optimal control strategy, it is not yet ready for implementation, as it remains in continuous time.  Furthermore, the modification to the base optimization algorithm is a function of $\bA$, which is unknown in practice.  Moreover, practical iterate averaging schemes in deep learning tend to use exponential instead of uniform weighting.  We thus discretize \eqref{eq:practical_control_strategy} and replace $\bA$ with a diagonal matrix $\bC$ to arrive at the following update strategy: given current parameters $\theta_k$ at iteration $k$, stochastic gradient $g_k$, and learning rate $\eta$, parameters are updated according to
\begin{align}\label{eq:stylized_update}
    \theta_{k+1} = \theta_k - \eta \cdot g_k +  \bC \left( \thetema[k] - \theta_k \right), \quad \thetema[k+1] = (1 - \beta) \cdot \thetema[k] + \beta \cdot \theta_{k+1},
\end{align}
where $\beta$ is the averaging parameter. Thus we are left with \eqref{eq:stylized_update}, which is a simple modification to the base optimization algorithm that can be implemented with minimal computational overhead. Indeed, $\bC$ is a diagonal matrix and thus the matmul can be accomplished coordinate wise, at roughly the same cost as the preconditioning step in Adam or AdamW \citep{loshchilov2017fixing,kingma2014adam}.   While there still exist several additional tweaks to \eqref{eq:stylized_update} required to make it work well in practice (\Cref{sec:method}), we have now derived a simple, practical modification to the base optimization algorithm that is motivated by the optimal control solution to the iterate-average estimator in a simplified setting.

\begin{algorithm}[t]
\caption{\textbf{\pace}: \textbf Pullback \textbf Averaging \textbf Control for \textbf Efficient Optimization}
\label{alg:controlled-bema}
\begin{algorithmic}[1]
\Require Learning rate $\eta$, pullback strength $c$, EMA power $\kappa$, scale $\varepsilon$, update frequency $\mathrm{uf}$.
\State $\thetema[0] \gets \theta_0$
\For{$t = 1, 2, \dots, T$}
    \State $\theta_t \gets \mathrm{AdamW.step}(\theta_{t-1};\, \eta)$ and $\hat v_t \gets \mathrm{AdamW.preconditioner}(\theta_t')$.
    \If{$t \mod \mathrm{uf} = 0$}
    \State $\lambda_{t,i} \gets \min\!\bigl(\eta\, c\, (1+t)^{-\kappa} / (\sqrt{\hat v_{t,i}} + \varepsilon),\ 1\bigr)$ \Comment{Clipped per-coordinate pullback gain}
    \State $\theta_t \gets \theta_t + \lambda_t \odot (\theta_{t-1}^{\mathrm{EMA}} - \theta_{t-1})$ 
    \State $\thetema \gets (1 - \beta_t)\, \thetema[t-1] + \beta_t\, \theta_t,\quad \beta_t = (1+t)^{-\kappa}$ \Comment{EMA update}
    \EndIf
\EndFor
\State \Return $\thetema[T]$
\end{algorithmic}
\end{algorithm}

\subsection{Convergence Guarantee in the Convex Setting}\label{ssec:convex_guarantees}

We derived \eqref{eq:stylized_update} above as a practical approximation to the optimal control strategy for the iterate-average estimator in a simplified setting through heuristics: while we provided a rigorous derivation of the optimal control strategy in the quadratic setting, we then made use of several approximations and discretization in order to arrive at the update rule.  In this section we provide formal guarantees for this update.  We first show that for arbitrary convex loss functions, the iterate-average estimator returned by \eqref{eq:stylized_update} converges at the same rate as SGD, up to a constant factor depending on $\bC$ and $\beta$.
\begin{theorem}\label{thm:convex_convergence}
    Suppose that $F: \rr^d \to \rr$ is a convex, $G$-Lipschitz function and suppose for each $k$ we have access to a stochastic gradient $g_k$ such that $\ee[g_k] = \nabla F(\theta_k)$ and $\ee[\norm{g_k - \nabla F(\theta_k)}^2] \leq \sigma^2$.  Suppose that $\theta_k$ evolves according to \eqref{eq:stylized_update} for some diagonal matrix $\bC$ with entries in $[0, 1]$ and some $\beta \in (0, 1)$. Let $\thetabar_T = T^{-1} \sum_{k=0}^{T-1} \theta_k$ be the flat average. Then it holds that for optimally tuned learning rate $\eta$,
    \iftoggle{colt}{$\ee\left[ F(\thetabar_T) - F(\mustar) \right] \leq \nicefrac{\sqrt{2 - \beta}}{\beta} \cdot \norm{\theta_1 - \mustar} \cdot \sqrt{\nicefrac{(G^2 + \sigma^2)}{T}}$.}{
    \begin{align}
        \ee\left[ F(\thetabar_T) - F(\mustar) \right] \leq \frac{\sqrt{2 - \beta}}{\beta} \cdot \norm{\theta_1 - \mustar} \cdot \sqrt{\frac{G^2 + \sigma^2}{T}}.
    \end{align}
    }
\end{theorem}
The proof of \Cref{thm:convex_convergence} is given in \Cref{app:convex_proofs} and follows by a standard analysis of SGD coupled with a recursive analysis to ensure that the update does not drift too far away from SGD itself.  In particular, \Cref{thm:convex_convergence} ensures that the iterate-average estimator returned by \eqref{eq:stylized_update} enjoys the same $O(1/\sqrt{T})$ and is worse only by a multiplicative factor depending on the strength of the \ema\ update.  Unfortunately, this issue is intrinsic to the analysis and likely cannot be improved\iftoggle{colt}{.}{ without additional assumptions.}

While the previous result shows that \eqref{eq:stylized_update} is never that much worse than SGD for convex functions, in the special case of quadratics, it can be quite a lot better.  As stated above, locally approximating the loss function of a LM by a quadratic is a common heuristic in optimization for deep learning \citep{block2025ema,block2024butterfly,zhang2019lookahead,zhang2019algorithmic}, where the Hessian of the loss is used as a proxy for local curvature.  In this setting, we show that \eqref{eq:stylized_update} enjoys improved convergence compared to SGD.
\begin{proposition}\label{prop:improved_convergence_quadratic}
    Suppose that $F(\theta) = \nicefrac 12 (\theta - \mustar)^\top \bA (\theta - \mustar)$ for some diagonal positive definite matrix $\bA$ and that $\theta_k$ evolves according to \eqref{eq:stylized_update} for some diagonal matrix $\bC$ with entries in $[0, 1]$ and some $\beta \in (0, 1)$.  Suppose further that the stochastic gradients are given by $g_k = \bA (\theta_k - \mustar) + \xi_k$, where $\xi_k$ is a zero-mean noise term with covariance $\sigma^2 \eye$.  Then for any learning rate $\eta$ such that $\normop{\eta \bA} < 1$ and choice of $\beta$, there is some $\bC$ such that $\lim_{T \to \infty} \ee\left[ \|\thetema[T] - \mustar\|^2 \right]$ is strictly smaller than the limiting squared error of the iterate-average estimator returned by SGD.
\end{proposition}
The proof of \Cref{prop:improved_convergence_quadratic} is given in \Cref{app:convex_proofs} and demonstrates that \eqref{eq:stylized_update} can rigorously improve the convergence of the iterate-average estimator in the quadratic setting.  
Finally, we show that the improvement over SGD can be arbitrarily large for certain choices of $\bC$ and $\beta$ for some problems.
\begin{proposition}\label{prop:arbitrarily_large_improvement}
    For any $\epsilon > 0$, there is some quadratic problem and some choice of $\bC$ and $\beta$ such that the limiting squared error of the iterate-average estimator returned by \eqref{eq:stylized_update} is at least a factor of $\epsilon$ smaller than the limiting squared error of the iterate-average estimator returned by SGD.
\end{proposition}

\section{Introducing \pace: Pullback Averaging Control for Efficient Optimization}
\label{sec:method}

\begin{figure}[t]
    \centering
    \includegraphics[width=\textwidth]{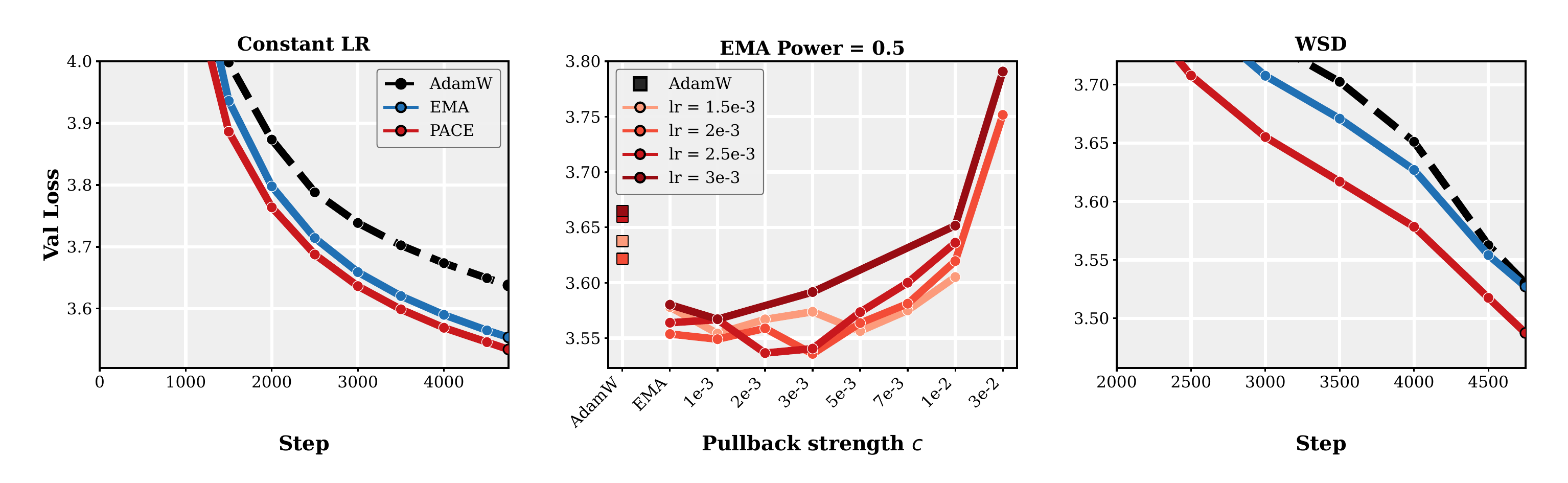}
    \caption{
        \textbf{Performance of \pace\ on pretraining} of GPT-2 (124M) on FineWeb at the Chinchilla-optimal token budget. \textbf{Left:} validation cross-entropy trajectories at a constant learning rate, with \ema\ and \pace\ at the same EMA power ($\kappa=0.5$) so that they differ only in the pullback. \textbf{Middle:} effect of pullback strength $c$ on validation cross-entropy at $\kappa=0.5$. \textbf{Right:} the analogous comparison under WSD, using linear learning rate decay to zero over the last $20\%$ of steps. \pace\ outperforms both the AdamW and \ema\ baselines under each schedule.
    }
    \label{fig:pretrain}
\end{figure}

While we derived a simple modification to the base optimization algorithm in \eqref{eq:stylized_update}, several additional tweaks make the proposed method work in practice.  The first key modification is that we will replace the SGD update with an AdamW update \citep{loshchilov2017fixing}, which is a standard optimizer choice for language models; in particular, we will use both momentum and preconditioning in the vanilla optimizer step.  The second question is how to choose the matrix $\bC$ that preconditions the pullback toward the iterate average.  According to the theory in \eqref{eq:practical_control_strategy}, the optimal choice of $\bC$ scales inversely with $\bA$ and $T$.  While in practice language models are obviously not quadratic, locally they can be approximated thereby, with $\bA$ corresponding to the Hessian \citep{zhang2019algorithmic,zhang2019lookahead}; following the standard noisy-quadratic-model identification of Adam's second-moment estimate $v_t$ as a Hessian-diagonal proxy, we will thus set $\bC$ to be a diagonal matrix with entries given by the Adam preconditioner \citep{kingma2014adam}.  One point where our empirical implementation deviates from the theory is that the latter suggests preconditioning should be scale like $\nicefrac 1{v_t}$, but for stability reasons we find that the standard Adam preconditioner of $\nicefrac 1{\sqrt{v_t} + \varepsilon}$ works better in practice.  

Following standard praxis in deep learning, in order to account for the nonstationarity of the optimization landscape, we replace a fixed $\beta$ in the EMA with a decaying $\beta_t = (1+t)^{-\kappa}$, where $\kappa \in (0, 1)$ is a hyperparameter \citep{busbridge2023scale,block2025ema,block2024butterfly,kaddour2022stop}; in effect this makes the iterate average more responsive to recent iterates, with $\kappa = 1$ recovering uniform averaging.  For the same reason, we replace the $T^{-1}$ scaling in \eqref{eq:practical_control_strategy} with a slower decaying factor that allows the pullback to remain significant even late in training.  We incorporate clipping to ensure that our updates are always convex combinations of the iterate averaged point and the na{\"i}ve optimizer update step. The resulting control becomes
\begin{align}\label{eq:practical_control_strategy_empirical}
    u_{t} = \min\left( (1 + t)^{- \kappa} \cdot \nicefrac{\eta c}{\left( \sqrt{v_{t}} + \epsilon \right)}, 1 \right)\left( \thetema[t] - \theta_t \right),
\end{align}
where the operations are applied coordinate-wise, $\eta$ is the learning rate, $c \geq 0$ is a single scalar setting the controller strength, and $v_t$ is the Adam second-moment estimate.  Note that in the limit as $c \uparrow \infty$, we recover the Lookahead optimizer with an inner loop of a single step \citep{zhang2019lookahead}, while at $c=  0$, we recover AdamW with \ema.  Finally, we allow for the update to only occur at a specified frequency, $\mathrm{uf}$, as is common in practical implementations of \ema. The result, \pace, is detailed in \cref{alg:controlled-bema}.  

\paragraph{Implementation Concerns.} In practice, there are two concerns with implementing \pace\ at scale.  First, we have introduced a new tuning parameter $c$ that controls the strength of the pullback toward the iterate average, as well as the update frequency $\mathrm{uf}$; on the other hand, the update frequency is already a parameter in standard implementations of \ema, and our experiments suggest that the optimal pullback strength is relatively robust and transfers across hyperparameters like learning rate (cf. \Cref{fig:135m_grid}).
Second, and perhaps more importantly, \pace\ requires maintaining an additional copy of model weights during training, which is a nontrivial memory overhead.  One mitigating factor is that this additional pressure is less relevant at large scale due to the use of large batch sizes, but exploring further mitigations of the memory overhead of \pace\ is an important direction for future work.

\begin{figure}[t]
    \centering
    \includegraphics[width=\textwidth]{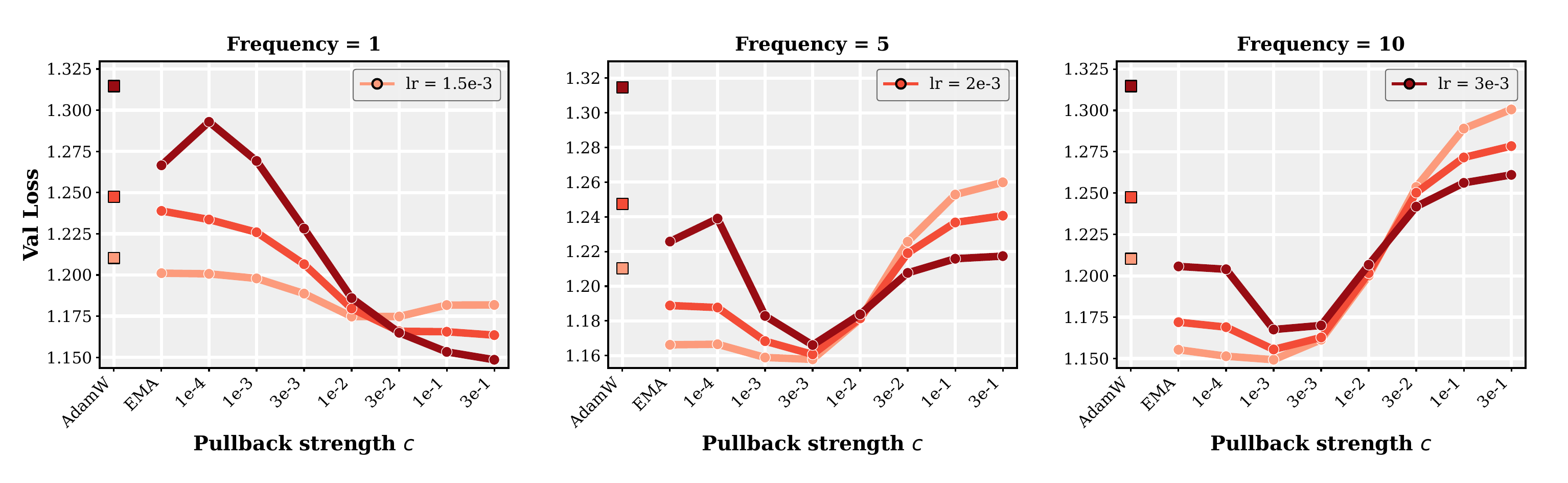}
    \caption{\textbf{Effect of update frequency and pullback strength.} Validation cross-entropy of SmolLM2-135M for different learning rates, pullback strengths, and update frequencies with $\kappa{=}0.3$.  Optimal pullback strength remains relatively robust to learning rate and update frequency, with improvement over the EMA baseline across a wide range of settings.}
    \label{fig:135m_grid}
\end{figure}

\section{Empirical Investigation of \pace}
\label{sec:experiments}

\begin{figure}[t]
    \centering
    \includegraphics[width=\textwidth]{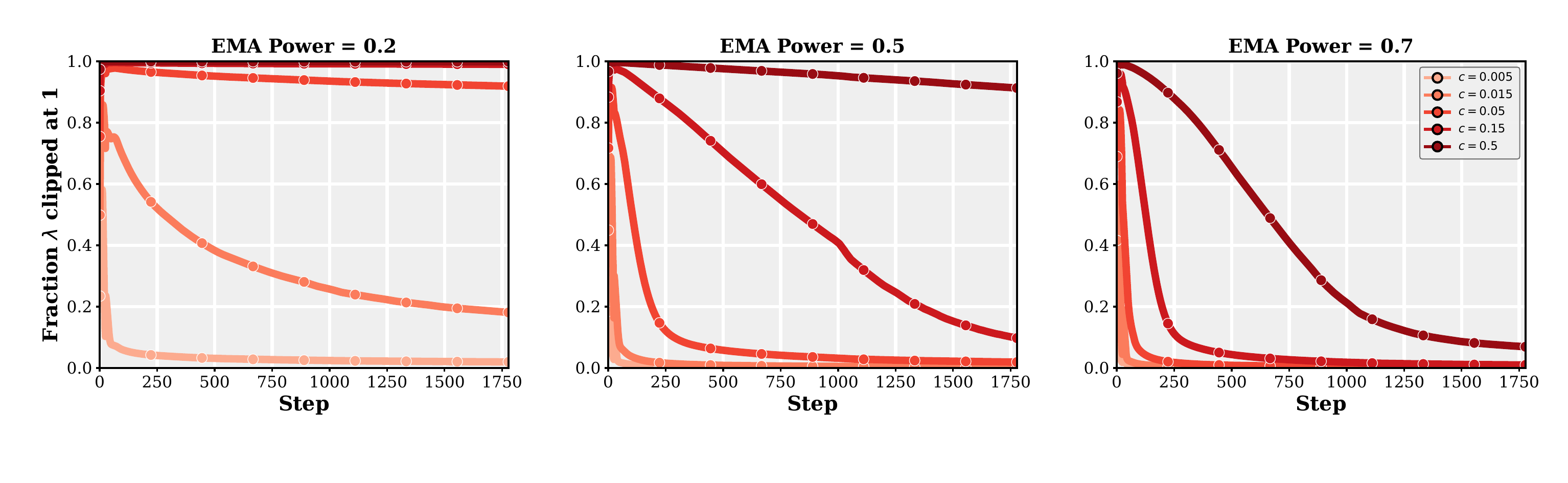}
    \caption{
        \textbf{Fraction of clipped updates} throughout fine-tuning of SmolLM2-1.7B at a fixed learning rate of $\eta=3{\times}10^{-4}$ for different pullback strengths $c$ and EMA powers $\kappa$.  For large $c$ and small $\kappa$, a substantial fraction of coordinates are clipped at $1$ throughout training, meaning that the pullback fully transports those coordinates to the EMA point, especially early in training.
        }
    \label{fig:clip_frac_body}
\end{figure}

Our proposed algorithm, \pace, was derived in a simplified quadratic model and it is thus not \emph{a priori} clear that it will be effective in practice on the non-convex, non-quadratic problems that arise in language model training.  In this section, we empirically evaluate \pace\ on two standard language-modeling regimes: supervised fine-tuning of openweight $1$--$2$B-parameter language models, and from-scratch pretraining of a $124$M-parameter model at the Chinchilla-optimal token budget.

Our experiments are primarily organized around three questions: (i) does \pace\ improve on vanilla AdamW and an EMA thereof across model families and learning rates? (ii) is the improvement robust to the optimizer's three new hyperparameters (pullback strength $c$, EMA power $\kappa$, and update frequency $\mathrm{uf}$)? and (iii) how does \pace\ compare to alternative baselines like learning rate decay and Schedule-Free \citep{defazio2024schedule}? Further ablations are deferred to \cref{sec:exp-details}.

\subsection{Empirical Setup}
\label{sec:exp-setup}

For \emph{post-training}, we finetune three pre-trained language models of similar scale, \textbf{SmolLM2-1.7B} \citep{smollm2}, \textbf{Qwen3-1.7B} \citep{qwen3}, and \textbf{Gemma3-1B} \citep{gemma3}, on the \texttt{smol-smoltalk} dataset \citep{smoltalk} for one epoch on the training split, plus a small-model dense ablation on \textbf{SmolLM2-135M}. For \emph{pretraining}, we train \textbf{GPT-2 (124M)} \citep{radford2019gpt2} on \texttt{FineWeb} \citep{penedo2024fineweb} at the Chinchilla-optimal token budget of $2.5$B tokens \citep{hoffmann2022chinchilla}. Fine-tuning runs use a single NVIDIA RTX PRO 6000 GPU; the GPT-2 pretraining sweep runs on a Google Cloud TPU v6e (Trillium) pod, with training hyperparameters and resource costs summarized in and around \cref{tab:models} of \cref{sec:exp-details}. We report final validation cross-entropy on a held-out \texttt{smol-smoltalk} split for fine-tuning and on a fixed FineWeb evaluation shard for pretraining.

\subsection{Main Results}
\label{sec:exp-main}

In \Cref{fig:train_curves}, we show the training curves of \pace, \ema, and vanilla AdamW on the fine-tuning task outlined above.  We see that \pace\ substantially improves on \ema\ at the same per-model learning rate.  Note that the comparison to AdamW without EMA is also included as an illustration.
 In \Cref{fig:pretrain}, we show the \iftoggle{colt}{}{same comparison for the} pretraining regime, where we again see a clear improvement of \pace\ over \ema.  
 In \Cref{fig:appendix_tulu_qwen} (deferred to the appendix), we repeat these experiments for Qwen on the Tulu dataset \citep{lambert2024tulu}, with similar results.
In all cases, we compare the best configurations of \ema\ and vanilla training to \pace, where we swept across learning rate, decay schedule, EMA power $\kappa$, and pullback strength $c$ to find the best configuration for each baseline.  Thus, these results suggest that \pace\ can significantly improve on the performance of the iterate-average estimator in practice, and that this improvement is robust across model families and datasets.

\subsection{Further Empirical Results}
\label{sec:exp-further}

In addition to our main results, we also conduct a number of ablations and further experiments to better understand the behavior of \pace\ and its sensitivity to its hyperparameters.  While many results are deferred to \cref{sec:exp-details} in the interest of space, we briefly summarize some of the key findings.

\paragraph{Effect of pullback strength $c$, EMA power $\kappa$, and update frequency.}  In each of our settings, we conduct a dense sweep over the pullback strength $c$ and EMA power $\kappa$ hyperparameters of \pace\ for several learning rates and update frequencies in order to understand the sensitivity of \pace\ to these hyperparameters and to verify that the improvement over EMA-evaluated AdamW is robust.  An example from our fine-tuning experiments is shown in \Cref{fig:135m_grid}, which is broadly replicated across models and settings (\Cref{fig:c_kappa,fig:appendix_135m_full_grid}).  In particular, we see that relatively small values of $c$ (e.g. $c=3 \times 10^{-3}$) can substantially improve on EMA-evaluated AdamW, and that the improvement is robust across a range of learning rates and update frequencies.

\paragraph{Effect of clipping $\lambda$.}  In \eqref{eq:practical_control_strategy_empirical} and \Cref{alg:controlled-bema}, we clip the pullback gain $\lambda$ to ensure that the update is always a convex combination of the iterate average and the na{\"i}ve optimizer step.  When the clipping is instantiated, \pace\ transports the relevant coordinate of the current iterate $\theta_k$ to exactly equal that of the EMA point $\thetema[k]$, recovering a special case of Lookahead \citep{zhang2019lookahead}.  In \cref{fig:clip_frac_body}, we show how the fraction of clipped coordinates varies with different hyperparameters throughout training on SmolLM2-1.7B at $\eta=3{\times}10^{-4}$, observing that in extreme cases, essentially all coordinates are clipped (especially early in training), but generally clipping is not enforced for most coordinates.  \iftoggle{colt}{}{
Further results are in \cref{fig:appendix_clip_frac,fig:appendix_clip_frac_vs_c}.
}

\paragraph{Effect of LR Decay.} 
While our theory assumed constant learning rate, decay schedules are a standard part of practical training pipelines.  We experiment with three decay schedules: (i) constant learning rate, (ii) linear decay to zero over the last $20\%$ of steps, and (iii) cosine decay to zero.  In  \Cref{fig:pretrain} (right) we compare \pace\ to AdamW with linear decay to zero over the last $20\%$ of steps in pretraining, finding that \pace\ strictly improves even in this setting.  Exhaustive results for all three schedules can be found in \Cref{sec:exp-details,sec:app-post}; in summary, \pace\ provides improvement over the baseline across all schedules, and the improvement is robust to the choice of schedule.

\paragraph{Comparing to the schedule-free baseline.} We also compare \pace\ to the most prominent schedule-free optimizer (SF) \citep{defazio2024schedule}, which also uses iterate averaging to stabilize training and remove the need for learning rate decay.  In \Cref{fig:bema_sf}, we compare \pace\ to Schedule-Free and observe competitive performance. While Schedule-Free requires one fewer copy of model weights than our na{\"i}ve \pace\ implementation, it must switch between `fast' and `slow' moving iterate averages during training, leading to slower training (12 vs 10 hours on a v6e TPU for Qwen3-1.7B in our experiments).  Finally, in \Cref{fig:bema_sf} (right), we show that \pace\ applied with a constant learning rate outperforms AdamW with WSD at every token budget, suggesting that \pace\ can be an effective alternative to learning rate decay.  Similar results hold for pretraining, as shown in \Cref{fig:appendix_pretrain_river}, providing further evidence that \pace\ can be an effective alternative to learning rate decay, although we emphasize that improved performance is typically achieved by combining \pace\ with a decay schedule.

\begin{figure}[t]
    \centering
    \includegraphics[width=\textwidth]{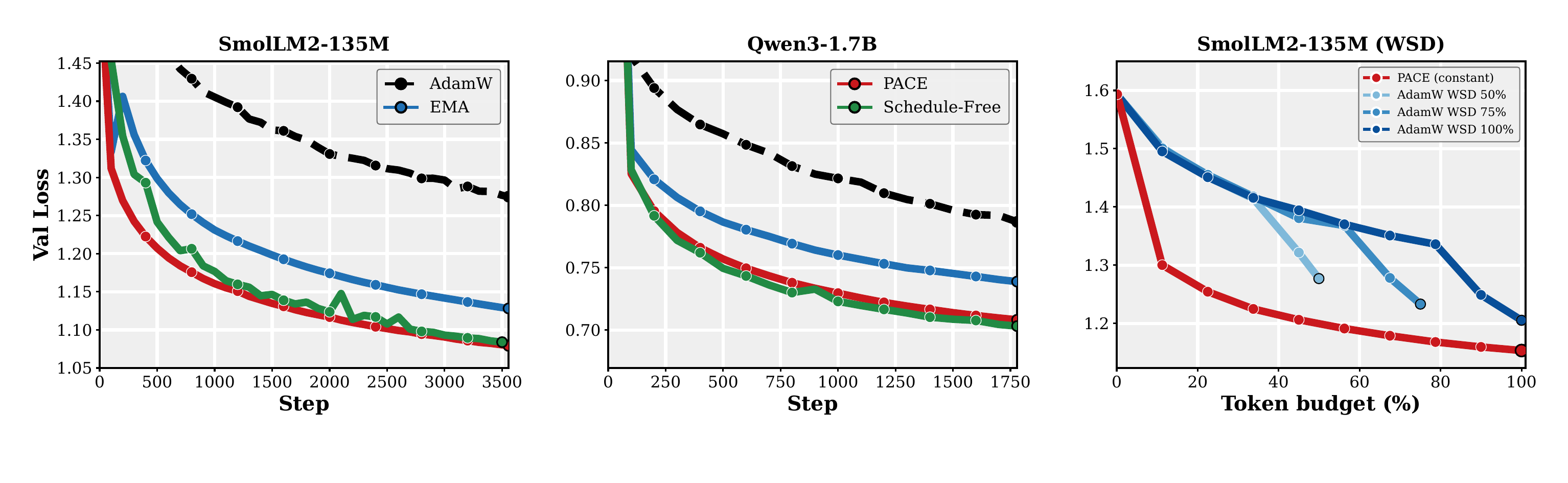}
    \caption{
        \textbf{Comparison of \pace\ to Schedule-Free \citep{defazio2024schedule}, AdamW, and EMA on fine-tuning.} Validation cross-entropy on \texttt{smol-smoltalk} for SmolLM2-135M (\textbf{left}), Qwen3-1.7B (\textbf{middle}), and a token-budget comparison on SmolLM2-135M (\textbf{right}): AdamW with WSD decays the learning rate to zero over the last $20\%$ of token budgets of $50\%/75\%/100\%$ of the run, while \pace\ trains at a constant learning rate over the full budget. \pace\ improves on AdamW and EMA-only and is competitive with Schedule-Free; in the right panel, every WSD branch ends above the \pace\ trajectory at every budget.
        }
    \label{fig:bema_sf}
\end{figure}

\section{Discussion}\label{sec:discussion}

In this work, we introduced \pace, a simple modification to existing LM optimization algorithms, formally derived in an optimal control formulation, that explicitly controls the training trajectory to improve the performance of the iterate-average estimator.  
We then showed empirically that \pace\ can improve training performance in both pre- and post-training settings.

\paragraph{Limitations.}  A theoretical limitation is the restrictiveness of the assumptions under which \pace\ was derived.  This is partly mitigated by the convergence guarantee in the convex setting, but rigorously incorporating momentum and adaptive preconditioning into the analysis of the optimizer is an important direction for future work.  
\iftoggle{colt}{}{We emphasize that in all experiments, \pace\ is implemented with both momentum and adaptive preconditioning.  }
There are several empirical limitations of the work as well.  First, due to resource constraints, we were unable to verify the efficacy of \pace\ at model scales larger than $2$B parameters for post-training and $124$M parameters for pretraining.  
Second, \pace\ requires maintaining an additional copy of model weights during training and a more detailed exploration of the memory overhead and mitigations thereof will be important for practical implications.  Finally, consistent with literature in LM optimization \citep{vyas2024soap,liu2023sophia,defazio2024schedule}, our headline runs use a single training seed for compute reasons; the multi-seed robustness studies in \Cref{fig:appendix_multiseed_robust} and \Cref{fig:appendix_pretrain_csweep} confirm that the method ordering is stable across seeds.

\paragraph{Future Work.} In addition to the directions for future work mentioned above, there are at least two additional questions raised by our work.  First, because our analysis is on SGD, our convergence guarantees do not require momentum and it is natural to explore the extent to which the pullback term alone can compensate for the lack of momentum; this would be a significant practical benefit of the method, as it would reduce the memory overhead of \pace\ by one copy of model weights.  Second, throughout our empirical study, we used AdamW as a base optimizer, but modern LM pretraining pipelines increasingly use second-order methods such as SOAP \citep{vyas2024soap} and MuON \citep{jordan2024muon}; exploring the extent to which \pace\ can be adapted to these optimizers and whether it provides similar benefits in that setting is an important direction for future work.  Moreover, it is natural to ask if these second-order preconditioners could be applied to the pullback term itself, which would be a more direct implementation of the optimal-control solution in \eqref{eq:practical_control_strategy}.

\section*{Acknowledgements}

We acknowledge the Google TPU Research Cloud for providing some of the compute resources used in this work.

\bibliographystyle{plainnat}
\bibliography{refs}

\clearpage
\tableofcontents

\newpage

\appendix

\section{Further Details on Empirical Setup}
\label{sec:exp-details}

In this appendix we collect the additional information needed to reproduce the experiments in \cref{sec:experiments}: training data, models, optimizer hyperparameters, per-experiment grids, hardware and compute, evaluation protocol, and the supplementary figures referenced from the body.

\paragraph{Training data.} We use \texttt{smol-smoltalk} \citep{smoltalk} as our default fine-tuning corpus, using the fixed train/validation split prepared for these experiments and packing examples to sequence length $1280$. Some supplementary sweeps use \texttt{tulu-3} \citep{lambert2024tulu} with the same supervised fine-tuning pipeline. For pretraining, we train GPT-2-124M on \texttt{FineWeb} \citep{penedo2024fineweb}.

\paragraph{Models.} \cref{tab:models} summarises the five model families we consider. The fine-tuning models are initialized from the official Hugging Face checkpoints. SmolLM2 and Gemma3 use their native tokenizers.

\begin{table}[t]
\centering
\caption{Per-run training budgets. Fine-tuning runs are one epoch unless noted; pretraining uses $2.5$B-tokens.}
\small
\begin{tabular}{lrlcc}
\toprule
Model & Params & Tokenizer & Seq.\ len & Eff.\ batch / steps \\
\midrule
SmolLM2-135M       & 135M  & SmolLM2          & 1280 & 128 / 889    \\
SmolLM2-1.7B       & 1.7B  & SmolLM2          & 1280 & 256 / 1{,}778 \\
Qwen3-1.7B         & 1.7B  & Qwen BPE         & 1280 & 256 / 1{,}779 \\
Gemma3-1B          & 1.0B  & Gemma3           & 1280 & 256 / 1{,}779 \\
GPT-2-124M         & 124M  & GPT-2 BPE        & 1024 & 512 / 4{,}750 \\
\bottomrule
\end{tabular}

\vspace{2pt}
\footnotesize The SmolLM2-135M method-comparison pool used in Fig.~\ref{fig:bema_sf} uses a long-budget rerun with 3{,}557 steps.

\label{tab:models}
\end{table}

\paragraph{Training.} Unless stated otherwise for the Schedule-Free and WSD baselines, AdamW-family runs use the standard defaults $(\beta_1, \beta_2) = (0.9, 0.999)$, $\varepsilon = 10^{-8}$, weight decay $\omega = 10^{-2}$, gradient clipping at $1.0$, and a $50$-step linear warmup followed by a constant learning rate (no decay). Training uses $\texttt{bfloat16}$ mixed precision. When needed, we use gradient accumulation to reach the reported effective batch size. Headline runs use seed $42$. The learning rate $\eta$ varies per experiment and is reported in \cref{tab:hp_per_experiment}.

\paragraph{Optimizer hyperparameters.} \pace\ adds three tunable hyperparameters on top of standard AdamW: the pullback strength $c$, the EMA power $\kappa$, and the update frequency $\mathrm{uf}$ (per-experiment values in \cref{tab:hp_per_experiment}). We use $\gamma=1.0$, $\rho=0$, and clamp the per-parameter (per-coordinate) pullback gain $\lambda_{t,i}$ at $1$, matching the implementation defaults. For Schedule-Free \citep{defazio2024schedule} we use the authors' reference implementation under their published NanoGPT recipe ($\beta_1{=}0.98$, $\beta_2{=}0.95$, weight decay $0.05$, no gradient clipping); learning rate is matched per model.

\begin{table}[h]
\centering
\caption{Sweep space used by the body and appendix figures.}
\footnotesize
\renewcommand{\arraystretch}{1.15}
\setlength{\tabcolsep}{5pt}
\begin{tabular}{@{}l l p{0.53\textwidth}@{}}
\toprule
\textbf{Figures} & \textbf{Cell} & \textbf{Sweep space} \\
\midrule
\ref{fig:train_curves}, \ref{fig:c_kappa}, \ref{fig:appendix_full_train_curves}
  & 1B fine-tuning
  & SmolLM2-1.7B: $\eta\!\in\!\{3,7,10\}{\times}10^{-4}$\newline
    Qwen3-1.7B: $\eta\!\in\!\{3,5\}{\times}10^{-4}$\newline
    Gemma3-1B: $\eta\!\in\!\{10^{-3},5{\times}10^{-3}\}$\newline
    Plotted $c$ bins $\in\!\{10^{-4},5{\times}10^{-4},10^{-3},3{\times}10^{-3},$\newline
    $\phantom{\text{Plotted }c\text{ bins }\in\!\{}10^{-2},3{\times}10^{-2},10^{-1},3{\times}10^{-1}\}$\newline
    $\kappa\!\in\!\{0.2,0.5,0.7\}$, $\mathrm{uf}{=}1$
  \\
\midrule
\ref{fig:appendix_tulu_qwen}
  & Qwen3-1.7B Tulu
  & $\eta\!\in\!\{3,5\}{\times}10^{-4}$, $\kappa\!\in\!\{0.2,0.5,0.7\}$\newline
    $c\!\in\!\{5{\times}10^{-2},10^{-1},1.5{\times}10^{-1},2.5{\times}10^{-1}\}$
  \\
\midrule
\ref{fig:clip_frac_body}, \ref{fig:appendix_clip_frac}, \ref{fig:appendix_clip_frac_vs_c}
  & SmolLM2-1.7B clipping
  & Same SmolLM2-1.7B fine-tuning runs as above; clipping telemetry is reported for $\eta\!\in\!\{3{\times}10^{-4},10^{-3}\}$, $\kappa\!\in\!\{0.2,0.5,0.7\}$, and $c\!\in\!\{5{\times}10^{-3},1.5{\times}10^{-2},5{\times}10^{-2},1.5{\times}10^{-1},5{\times}10^{-1}\}$.
  \\
\midrule
\ref{fig:135m_grid}, \ref{fig:appendix_135m_full_grid}
  & SmolLM2-135M
  & Dense sweep: $\eta\!\in\!\{1,\;1.5,\;2,\;3,\;5\}{\times}10^{-3}$\newline
    $c\!\in\!\{0,\;10^{-4},\;10^{-3},\;3{\times}10^{-3},\;10^{-2},$\newline
    $\phantom{c\!\in\!\{}3{\times}10^{-2},\;10^{-1},\;3{\times}10^{-1}\}$\newline
    $\kappa\!\in\!\{0.3,\,0.5,\,0.7\}$, $\mathrm{uf}\!\in\!\{1,\,5,\,10\}$
  \\
\midrule
\ref{fig:bema_sf}, \ref{fig:appendix_train_curves}, \ref{fig:appendix_wsd_val_curves}, \ref{fig:appendix_wsd_finetune}
  & Method and WSD comparisons
  & SmolLM2-135M: AdamW/EMA/\pace\ at $\eta\!\in\!\{10^{-3},5{\times}10^{-3}\}$;\newline
    Schedule-Free: $\eta\!\in\!\{10^{-3},\;3{\times}10^{-3},\;5{\times}10^{-3},\;10^{-2}\}$\newline
    WSD baselines use $\eta\!\in\!\{1,1.5,2,3,5\}{\times}10^{-3}$\newline
    Qwen3-1.7B: AdamW/EMA/\pace\ at $\eta\!\in\!\{3,5\}{\times}10^{-4}$; Schedule-Free at $\eta\!\in\!\{3{\times}10^{-4},10^{-3}\}$.
  \\
\midrule
\ref{fig:pretrain}, \ref{fig:appendix_pretrain_kappa_sweep}, \ref{fig:appendix_pretrain_csweep}
  & GPT-2-124M
  & $\eta\!\in\!\{3{\times}10^{-4},\; 10^{-3},\; 1.5{\times}10^{-3},\; 2{\times}10^{-3},$\newline
    $\phantom{\eta\!\in\!\{}2.5{\times}10^{-3},\; 3{\times}10^{-3},\; 10^{-2}\}$\newline
    $c$: coarse grid $10^{-4}$--$3{\times}10^{-1}$ plus local grid $1{\times}10^{-3}$--$10^{-2}$\newline
    $\kappa\!\in\!\{0.2,\,0.5,\,0.7\}$, $\mathrm{uf}{=}1$
  \\
\midrule
\ref{fig:appendix_fig1_all_schedules}, \ref{fig:appendix_multiseed_robust}
  & 1B schedules and seeds
  & Schedules $\in\!\{$constant, cosine, WSD$\}$ for AdamW/\ema\ at each model's headline $\eta$ ($10^{-3}$ for SmolLM2-1.7B and Gemma3-1B, $5{\times}10^{-4}$ for Qwen3-1.7B); \pace\ at a constant $\eta$.\newline
    Multi-seed: SmolLM2-1.7B at $\eta{=}10^{-3}$, seeds $\{42,1337,2024\}$.
  \\
\midrule
\ref{fig:bema_sf}
  & SmolLM2-135M budgets
  & AdamW WSD at token budgets $\{444,667,889\}$ steps,\newline
    $\eta\!\in\!\{3,5\}{\times}10^{-3}$; \pace\ at a constant learning rate over the full budget ($c{=}10^{-1}$, $\kappa{=}0.3$).
  \\
\midrule
\ref{fig:pretrain}, \ref{fig:appendix_pretrain_csweep}, \ref{fig:appendix_pretrain_csweep_kappa}
  & GPT-2-124M schedules
  & Shared grid per schedule (constant, cosine, WSD) at $\eta{=}2{\times}10^{-3}$:\newline
    $c\!\in\!\{3{\times}10^{-4},10^{-3},2{\times}10^{-3},3{\times}10^{-3},5{\times}10^{-3},7{\times}10^{-3},10^{-2}\}$\newline
    $\kappa\!\in\!\{0.1,0.2,0.3,0.5\}$, plus AdamW and \ema, each trained at three seeds $\{0,1,42\}$.
  \\
\midrule
\ref{fig:appendix_pretrain_river}
  & GPT-2-124M budgets
  & WSD at token budgets $\{2375,3563,4750\}$ steps\newline
    ($\eta{=}2{\times}10^{-3}$, $\kappa{=}0.5$, $c{=}3{\times}10^{-3}$); the token-budget overlay uses \pace\ at a constant learning rate ($c{=}2{\times}10^{-3}$, $\kappa{=}0.5$).
  \\
\bottomrule
\end{tabular}

\label{tab:hp_per_experiment}
\end{table}

\paragraph{Hardware and compute.} Fine-tuning runs used NVIDIA RTX PRO 6000 Blackwell (98\,GB) cards; the GPT-2 pretraining sweep used Google Cloud TPU v6e (Trillium) chips. Per-run budgets are listed in \cref{tab:models}.

\paragraph{Evaluation.} For supervised fine-tuning we report held-out cross-entropy on the fixed validation split for each dataset. For pretraining we report cross-entropy on a fixed FineWeb validation shard. EMA-family methods are evaluated using $\theta_t^{\mathrm{EMA}}$ rather than the live iterate (we swap weights at evaluation time and restore them afterwards); plain AdamW reports its live-weight value; Schedule-Free reports its native vanilla evaluation since SF does not maintain a separate estimator.

\paragraph{Reproducibility.} We will release the code and configuration files needed to reproduce the reported sweeps.

\subsection{Supplementary figures}
\label{sec:exp-details-figs}

This subsection collects the full hyperparameter-sweep and figures
referenced from \cref{sec:experiments} and the result appendices; takeaways are
stated in the captions and discussed where each figure is cited.

\begin{figure}[t]
  \centering
  \includegraphics[width=\textwidth]{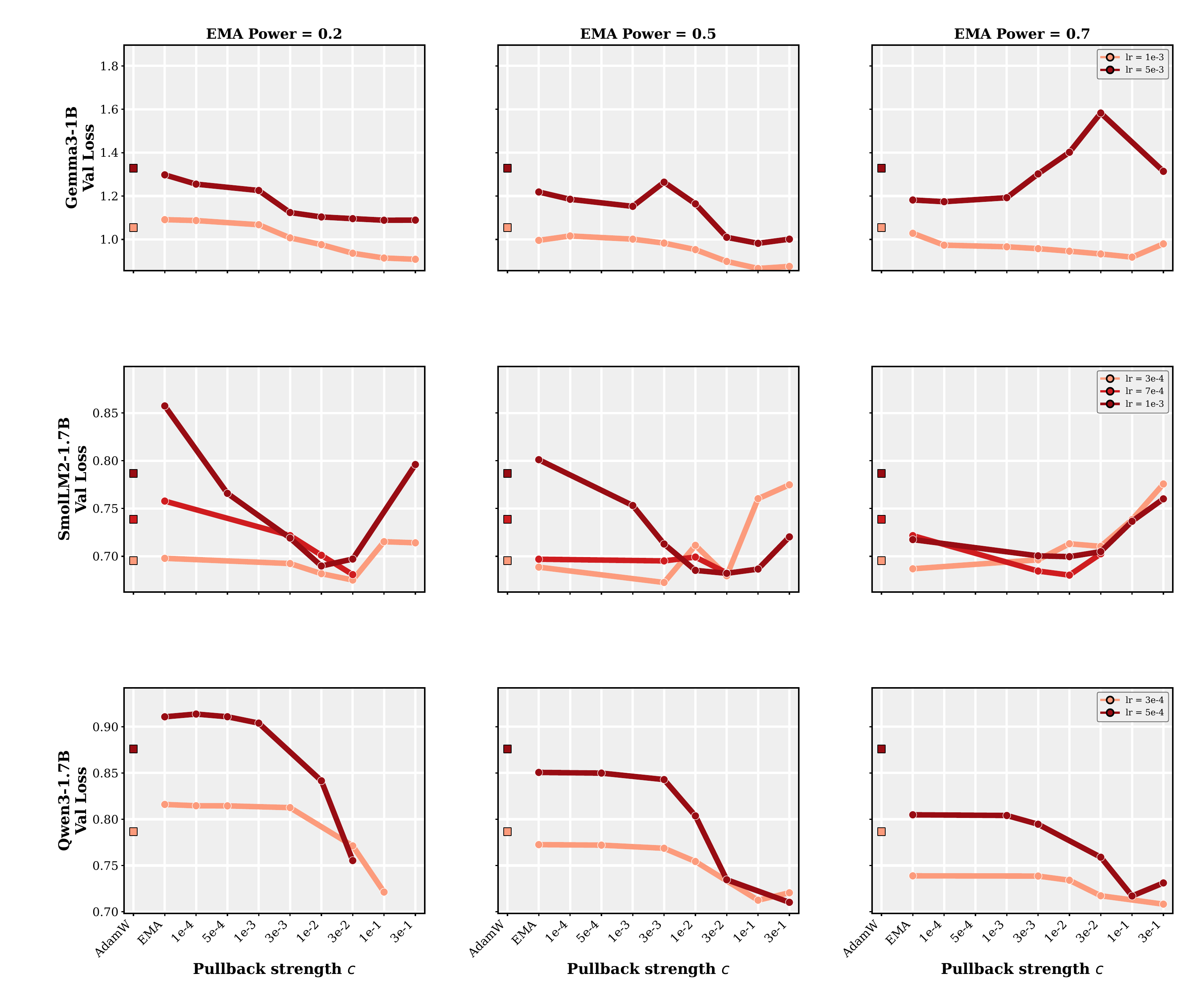}
  \captionof{figure}{\textbf{Effect of EMA power and pullback strength on SmolLM2-1.7B, Qwen3-1.7B, and Gemma3-1B.} Validation cross-entropy across learning rates, pullback strengths, and EMA powers. Optimal pullback strength remains relatively robust to learning rate and EMA power, with improvement over the EMA baseline across a wide range of settings and models.}
  \label{fig:c_kappa}
\end{figure}

\begin{figure}[ht]
  \centering
  \includegraphics[width=0.667\textwidth]{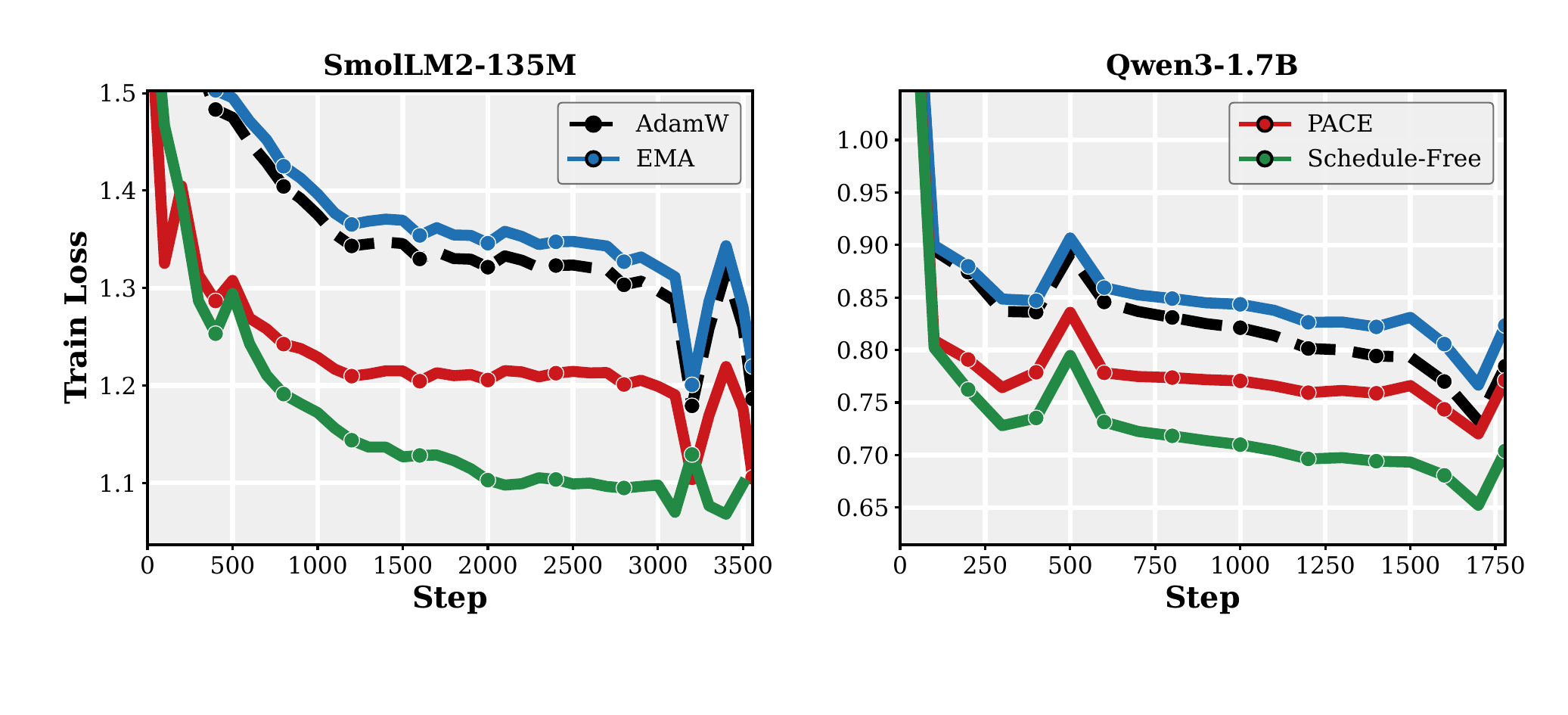}
  \caption{\textbf{Training-loss comparison of \pace\ to Schedule-Free \citep{defazio2024schedule}, AdamW, and EMA on fine-tuning.} Smoothed training loss on the same model selections as Fig.~\ref{fig:bema_sf}. Although \pace\ has higher training loss than Schedule-Free in this view, its validation loss is competitive in Fig.~\ref{fig:bema_sf}.}
  \label{fig:appendix_train_curves}
\end{figure}

\begin{figure}[ht]
  \centering
  \includegraphics[width=\textwidth]{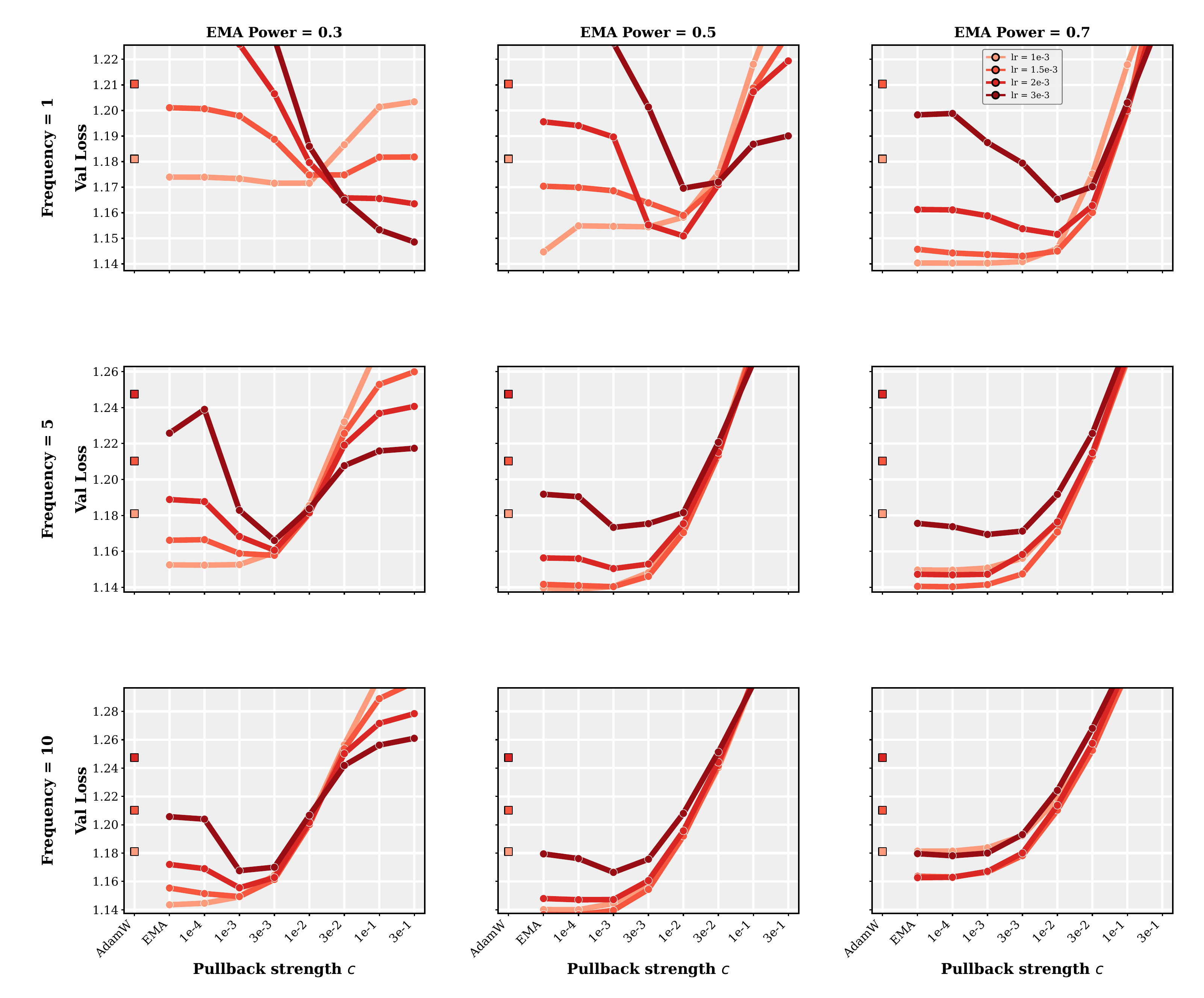}
  \caption{\textbf{Effect of update frequency, EMA power, and pullback strength.} Validation cross-entropy of SmolLM2-135M for different learning rates, pullback strengths, EMA powers, and update frequencies. Optimal pullback strength remains relatively robust to learning rate and update frequency, with improvement over the EMA baseline across a wide range of settings.}
  \label{fig:appendix_135m_full_grid}
\end{figure}

\begin{figure}[ht]
  \centering
  \includegraphics[width=\textwidth]{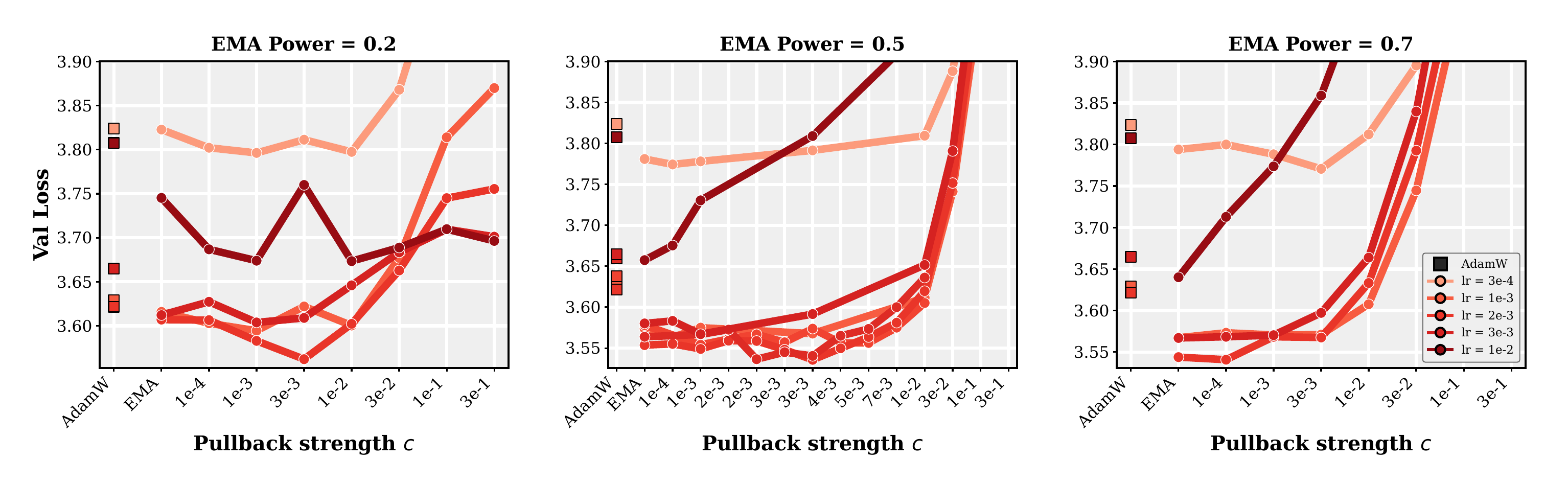}
  \caption{\textbf{Effect of EMA power and pullback strength on pretraining.} Validation cross-entropy on GPT-2-124M trained on FineWeb for different learning rates, pullback strengths, and EMA powers. Optimal pullback strength remains robust, with improvement over the EMA baseline across a wide range of settings.}
  \label{fig:appendix_pretrain_kappa_sweep}
\end{figure}

\begin{figure}[ht]
  \centering
  \includegraphics[width=\textwidth]{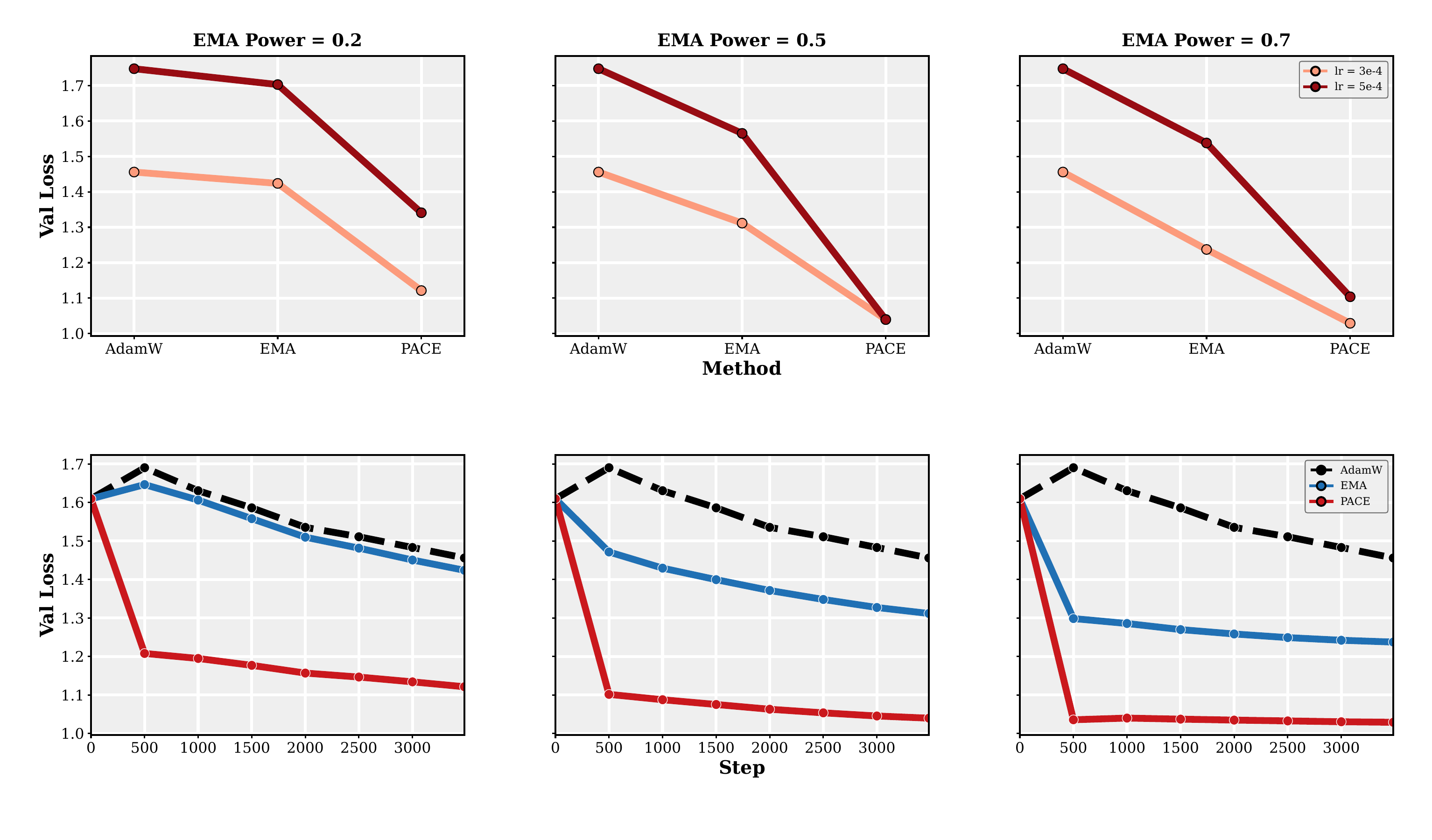}
  \captionof{figure}{\textbf{Effect of EMA power and pullback strength on Qwen3-1.7B Tulu-3 fine-tuning.} Validation cross-entropy across learning rates and EMA powers. \textbf{Top:} final validation loss for AdamW, EMA, and \pace. \textbf{Bottom:} validation trajectories at $\eta=3{\times}10^{-4}$. \pace\ improves over the matched AdamW and EMA baselines across the reported settings.}
  \label{fig:appendix_tulu_qwen}
\end{figure}

\begin{figure}[ht]
  \centering
  \includegraphics[width=\textwidth]{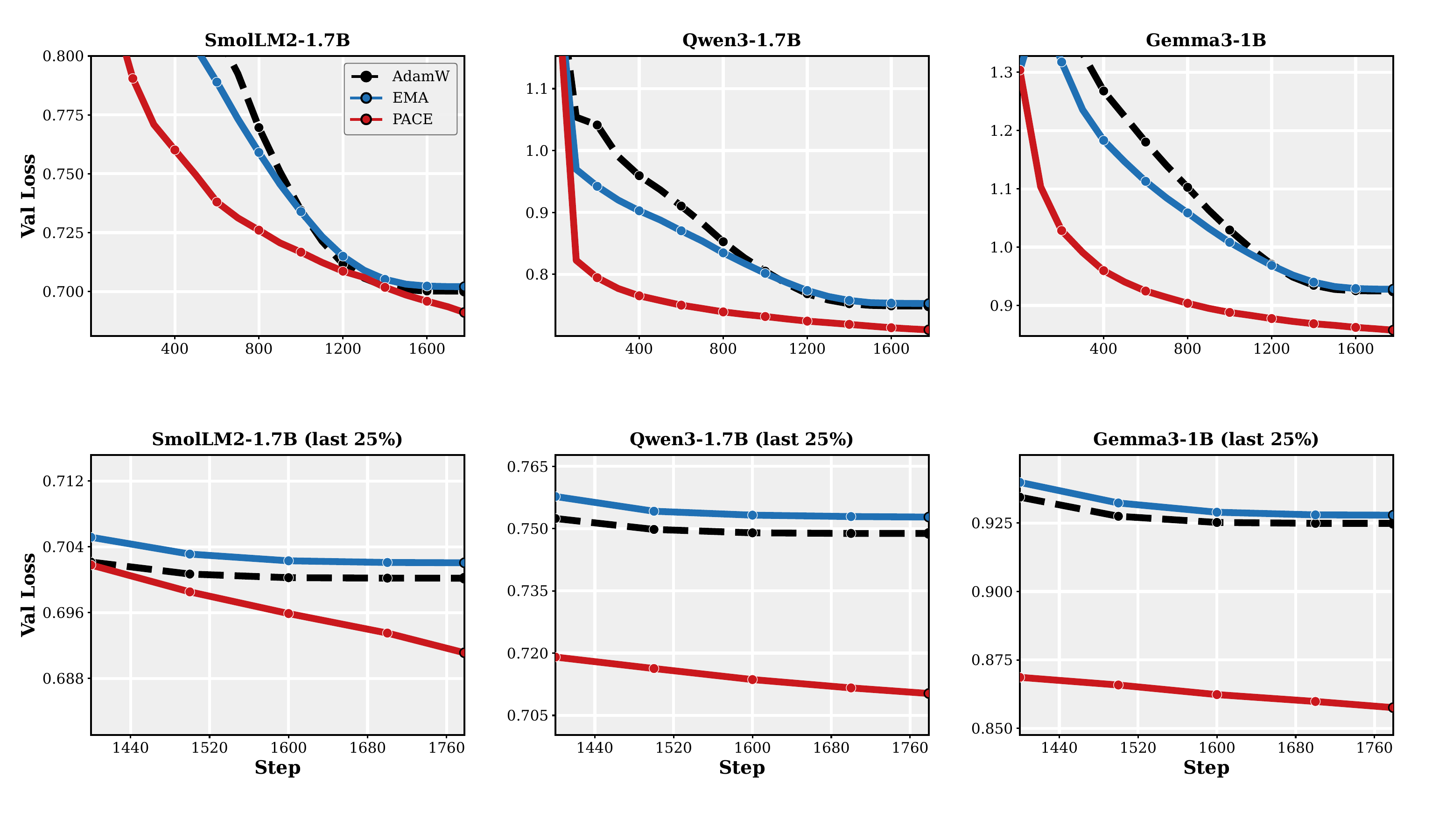}
  \caption{\textbf{Last 25\% of the cross-entropy trajectories from Fig.~\ref{fig:train_curves}.} The top row shows the full validation trajectories, and the bottom row zooms into the final quarter of training. \pace\ shows a clear improvement on SmolLM2-1.7B, Qwen3-1.7B, and Gemma3-1B.}
  \label{fig:appendix_full_train_curves}
\end{figure}

\begin{figure}[ht]
  \centering
  \includegraphics[width=\textwidth]{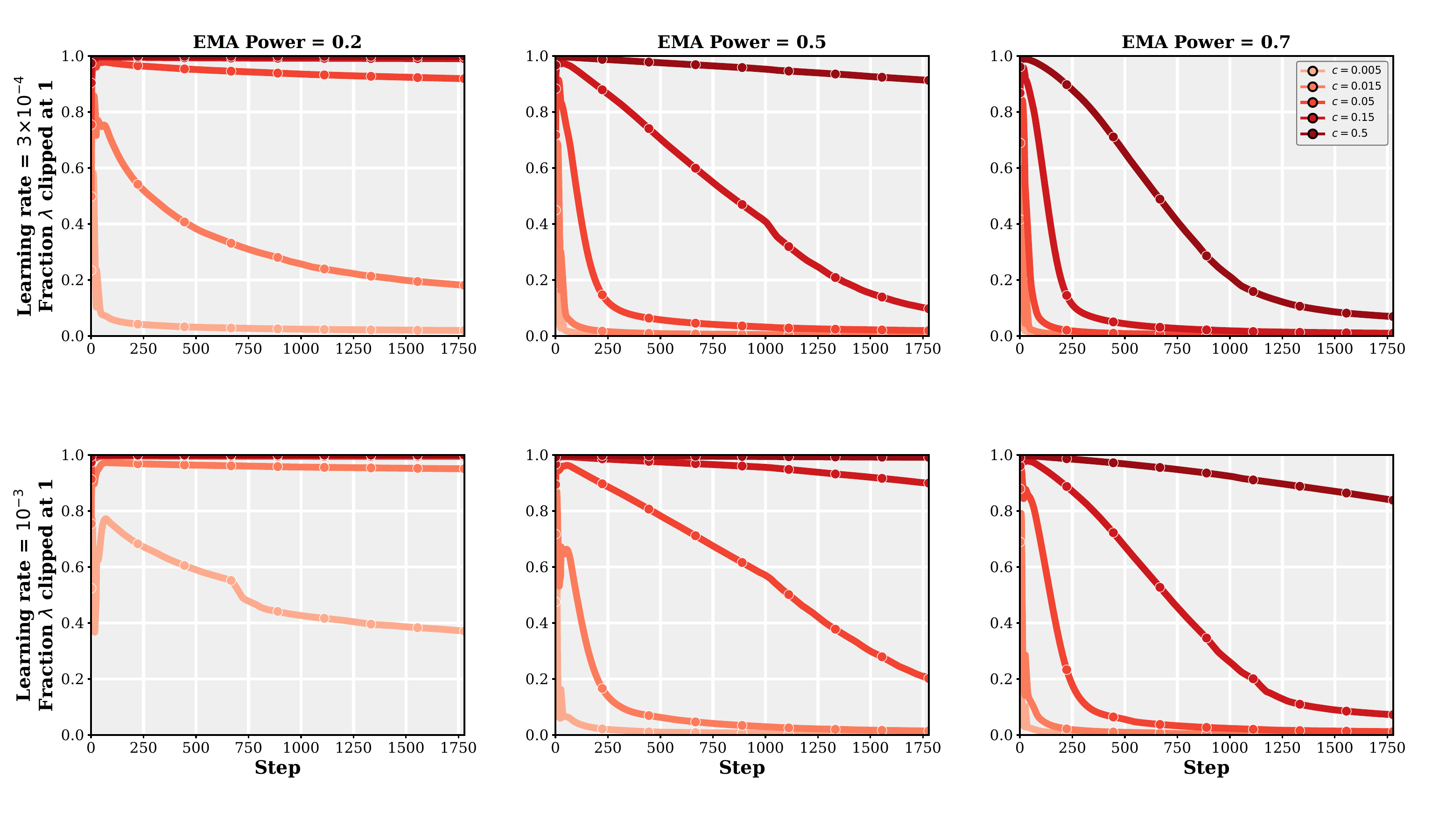}
  \caption{\textbf{Fraction of clipped updates over training.} Fraction of parameters with saturated pullback gain $\lambda_{t,i}=1$ on SmolLM2-1.7B, across learning rate, EMA power, and pullback strength. Saturation increases with pullback strength and is most pronounced for smaller EMA power.}
  \label{fig:appendix_clip_frac}
\end{figure}

\begin{figure}[ht]
  \centering
  \includegraphics[width=\textwidth]{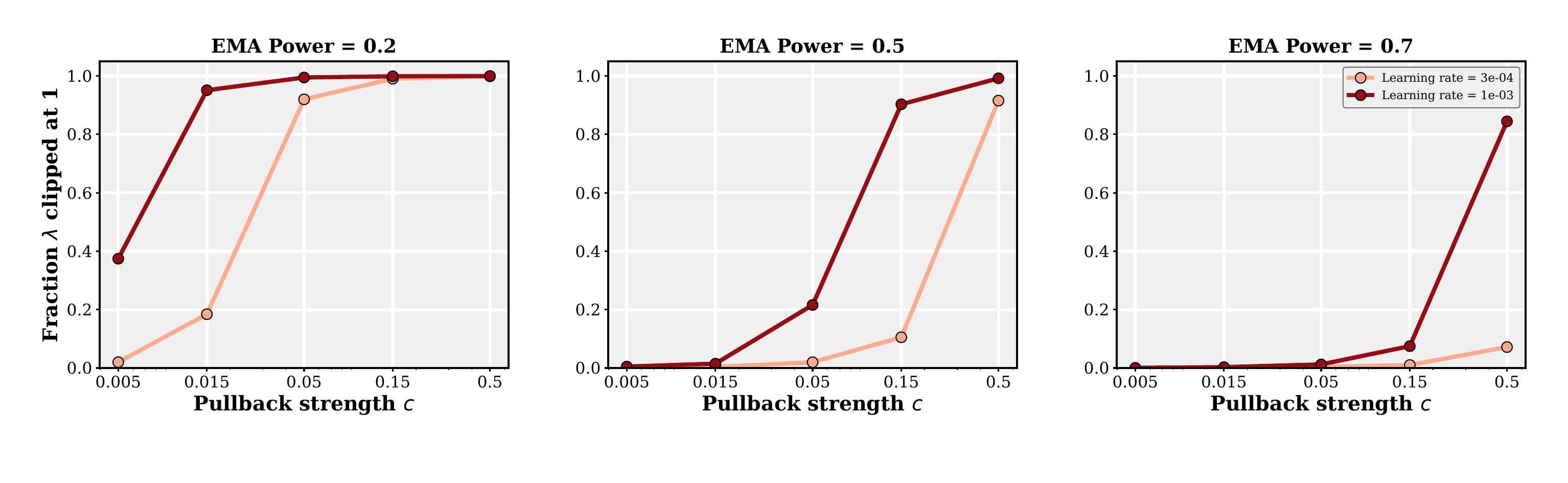}
  \caption{\textbf{End-of-training fraction of clipped updates.} Fraction of parameters with saturated pullback gain $\lambda_{t,i}=1$ on SmolLM2-1.7B as a function of pullback strength and learning rate. Smaller EMA power reaches saturation at smaller pullback strengths.}
  \label{fig:appendix_clip_frac_vs_c}
\end{figure}

\begin{figure}[ht]
  \centering
  \includegraphics[width=\textwidth]{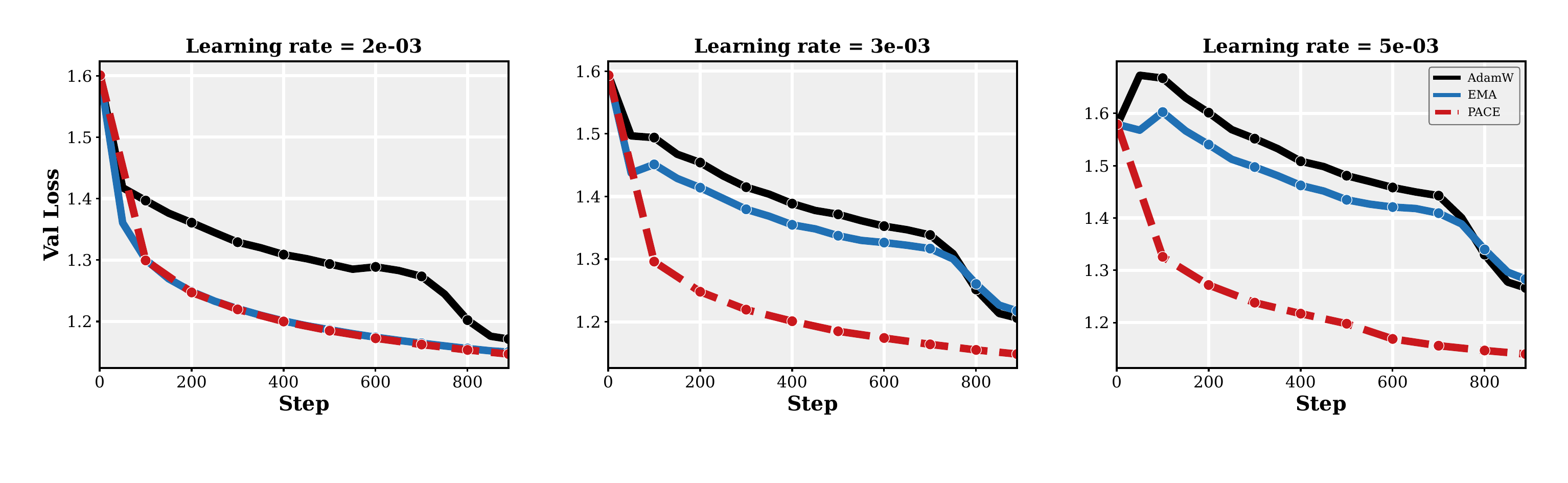}
  \caption{\textbf{\pace\ against AdamW and EMA under fine-tuning.} Validation cross-entropy on SmolLM2-135M at three learning rates, with the baselines given their best decay schedule and \pace\ held at a constant learning rate. \pace\ ends below both baselines at every learning rate, and its margin over them grows sharply as the learning rate increases.}
  \label{fig:appendix_wsd_val_curves}
\end{figure}

\section{Additional Results}
\label{sec:app-results}

\subsection{Fine-tuning}
\label{sec:app-post}

The advantage of \pace\ over AdamW and \ema\ is robust: it holds whatever
learning-rate schedule the baselines are given
(\Cref{fig:appendix_fig1_all_schedules}) and across three random seeds
(\Cref{fig:appendix_multiseed_robust}).

\paragraph{Baselines across all learning-rate schedules.}
Even when AdamW and \ema\ are given their best schedule---constant, cosine, or
WSD---\pace\ at a constant learning rate stays below both, on all three models
(\Cref{fig:appendix_fig1_all_schedules}).

\begin{figure}[htb]
  \centering
  \includegraphics[width=0.676\textwidth]{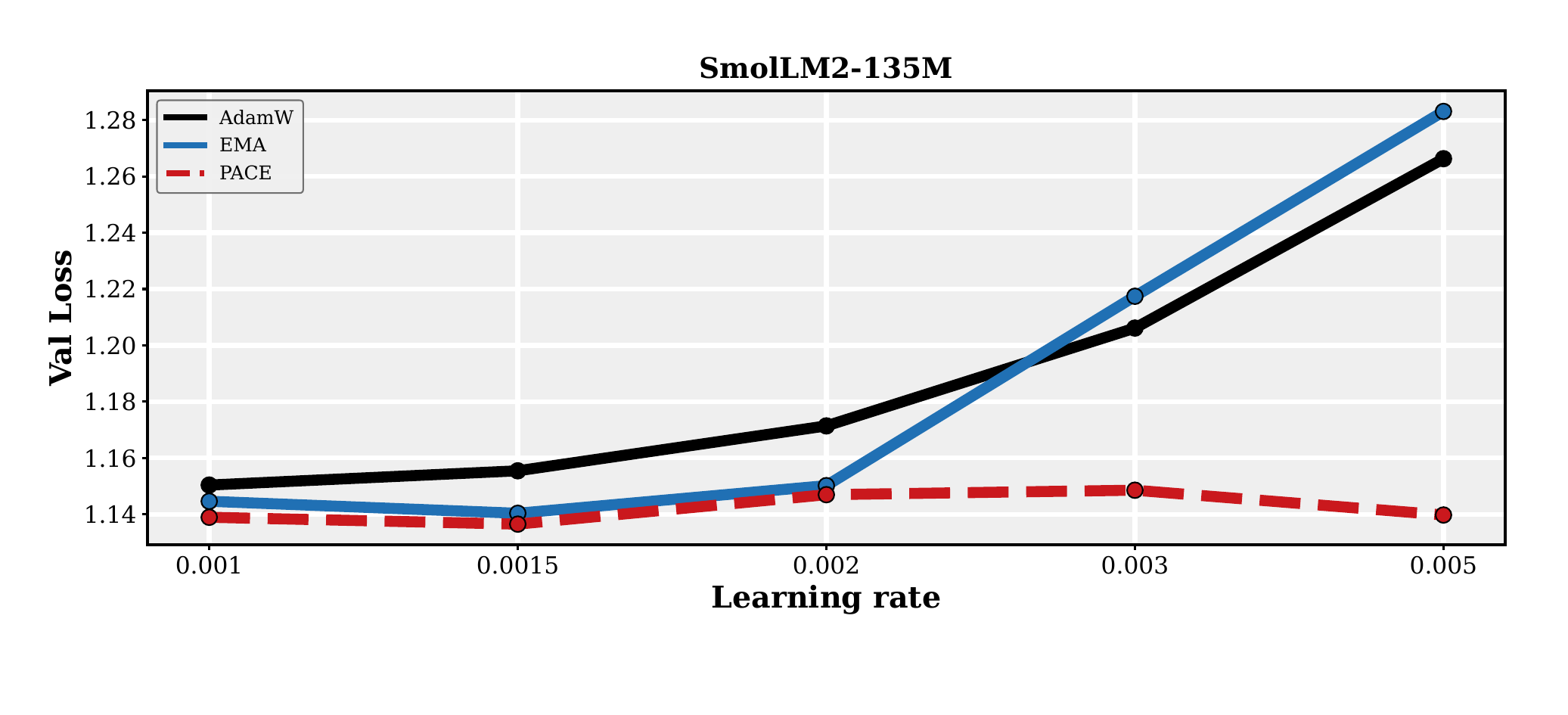}
  \caption{\textbf{Fine-tuning across learning rates.} Best validation cross-entropy on SmolLM2-135M for AdamW, EMA, and \pace, with the baselines given their best decay schedule and \pace\ held at a constant learning rate. \pace\ attains the lowest loss at every learning rate, by a margin that widens as the learning rate increases and the baselines degrade.}
  \label{fig:appendix_wsd_finetune}
\end{figure}

\begin{figure}[ht]
  \centering
  \includegraphics[width=\textwidth]{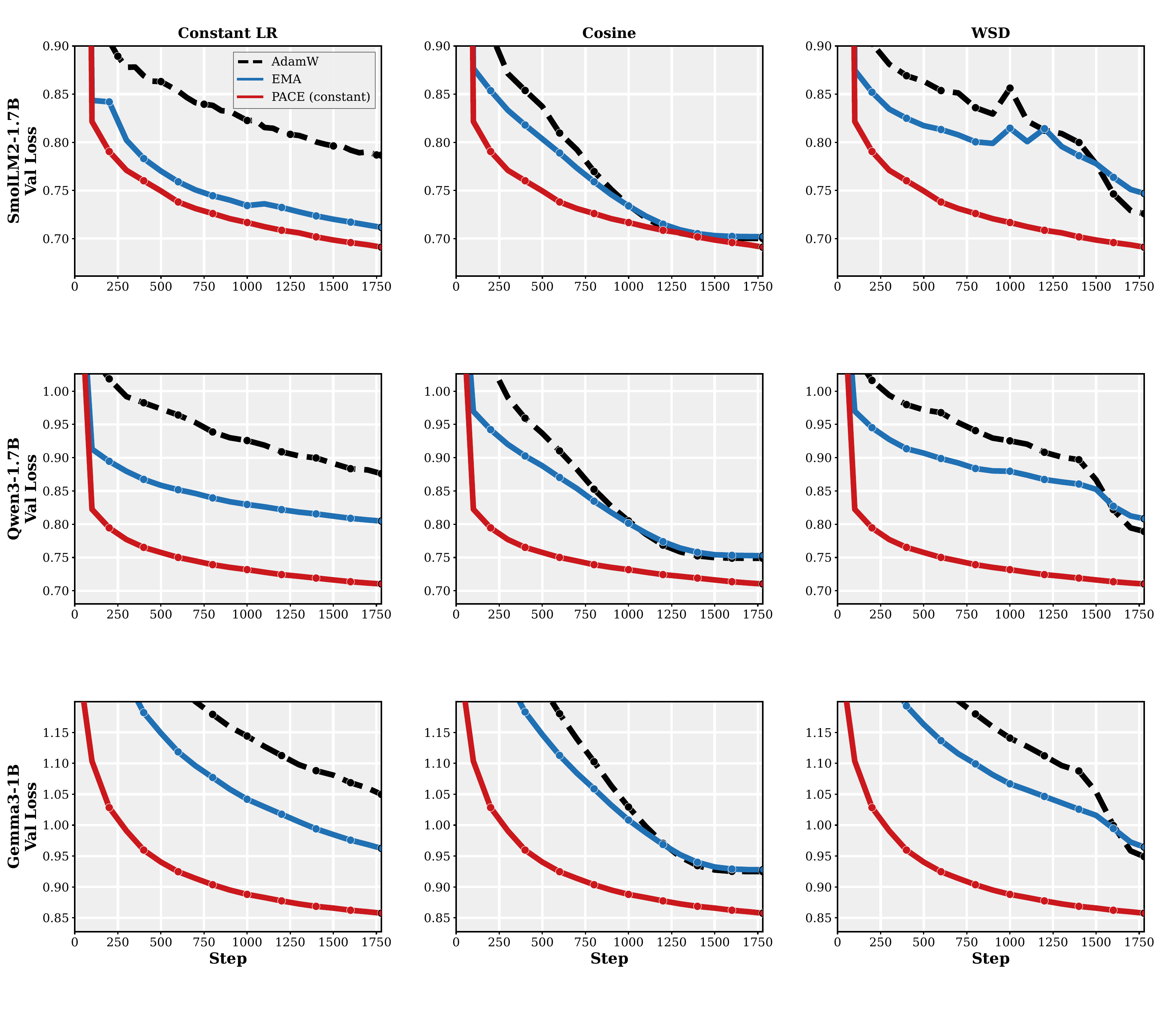}
  \caption{\textbf{\pace\ against the baselines under every learning-rate
  schedule.} Validation cross-entropy on \texttt{smol-smoltalk}; rows are
  models, columns are the baselines' learning-rate schedule (constant, cosine,
  WSD), with the \pace\ constant-learning-rate run as the reference in each
  panel. \pace\ improves on both baselines under every schedule on all three
  models.}
  \label{fig:appendix_fig1_all_schedules}
\end{figure}

\paragraph{Multi-seed robustness.} Across three seeds, the entire \pace\ band
stays below those of AdamW and \ema\ at every step---under both a constant
learning rate and WSD, on all three models---so the improvement exceeds the
seed-to-seed spread (\Cref{fig:appendix_multiseed_robust}).

\begin{figure}[t]
  \centering
  \includegraphics[width=\textwidth]{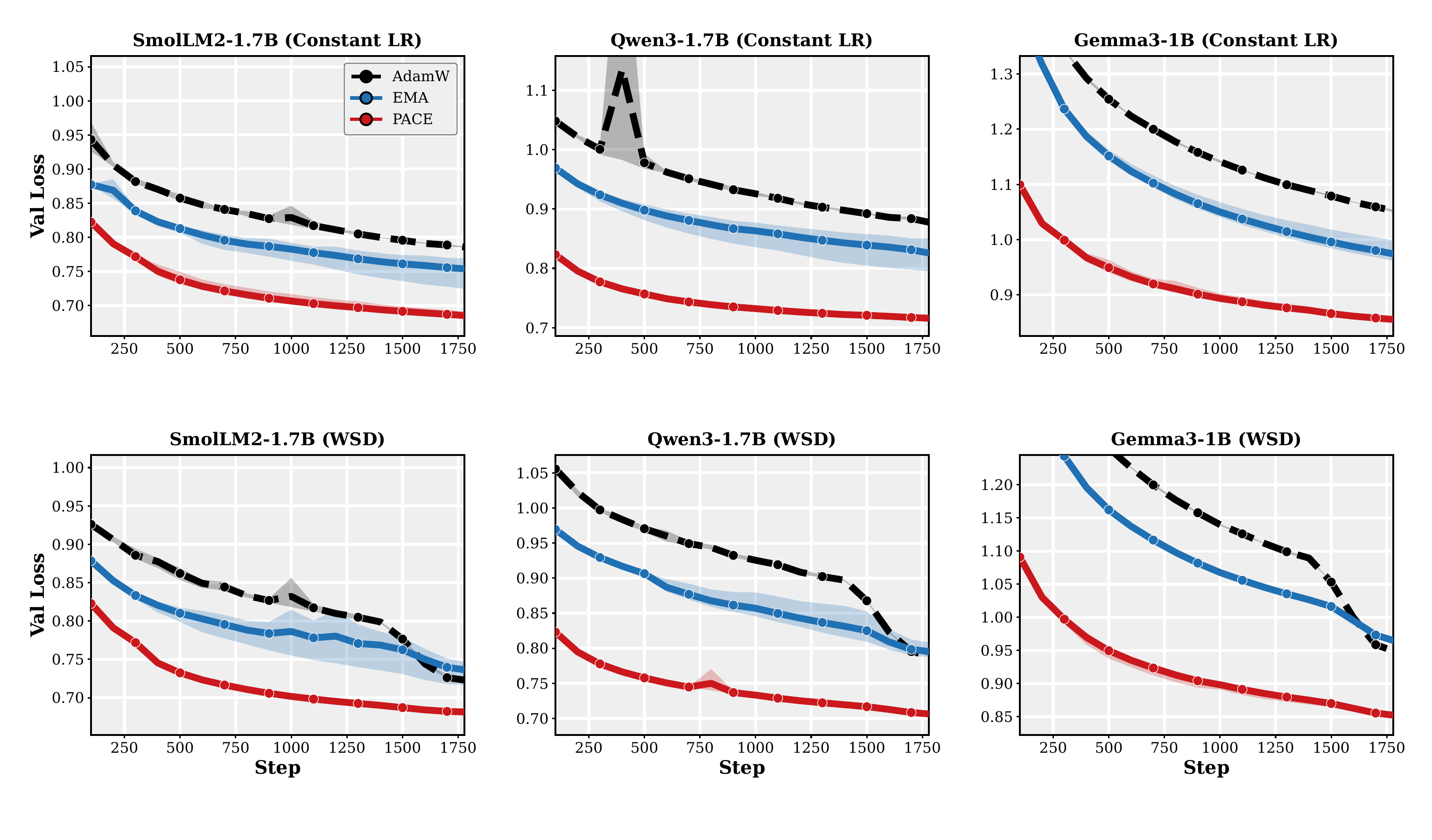}
  \caption{\textbf{Multi-seed robustness of fine-tuning.} Across three seeds,
  \pace\ stays below both AdamW and \ema\ on all three models, under both a
  constant learning rate and WSD, by more than the seed-to-seed variation at
  every step.}
  \label{fig:appendix_multiseed_robust}
\end{figure}

\subsection{Pretraining}
\label{sec:app-pretrain}

In pretraining, \pace\ beats both the AdamW and \ema\ baselines under every
learning-rate schedule, and the best pullback strength is the same under all of
them. The figures below show the full pullback-strength sweep on GPT-2-124M
(trained from scratch on \texttt{FineWeb} at the Chinchilla-optimal budget), as
the mean over three seeds $\{0,1,42\}$.

\paragraph{Multi-seed pullback-strength sweep across schedules.}
\Cref{fig:appendix_pretrain_csweep} sweeps the pullback strength $c$ under all
three schedules (the AdamW and best-$\kappa$ \ema\ baselines are the two leftmost
ticks). Two things stand out. First, at its optimum \pace\ beats both baselines
under every schedule, and the improvement holds across a band of pullback
strengths ($10^{-3}\le c\le 3{\times}10^{-3}$)---wider still under WSD and cosine.
Second, the optimum is the same ($c\approx 3{\times}10^{-3}$) under all three
schedules, so the choice of $c$ transfers. WSD with \pace\ is best overall
($3.507\pm0.006$, below the constant-LR optimum of $3.537\pm0.004$ by several
times the seed spread), and \pace\ at a constant learning rate comes within
$0.006$ of AdamW under WSD ($3.537$ vs.\ $3.531$) with no schedule at all.
\Cref{fig:appendix_pretrain_csweep_kappa} resolves the same grid by the EMA
power $\kappa$ (at seed $42$) and shows a basin around the optimal $c$ at every
$\kappa$.

\begin{figure}[ht]
  \centering
  \includegraphics[width=0.676\textwidth]{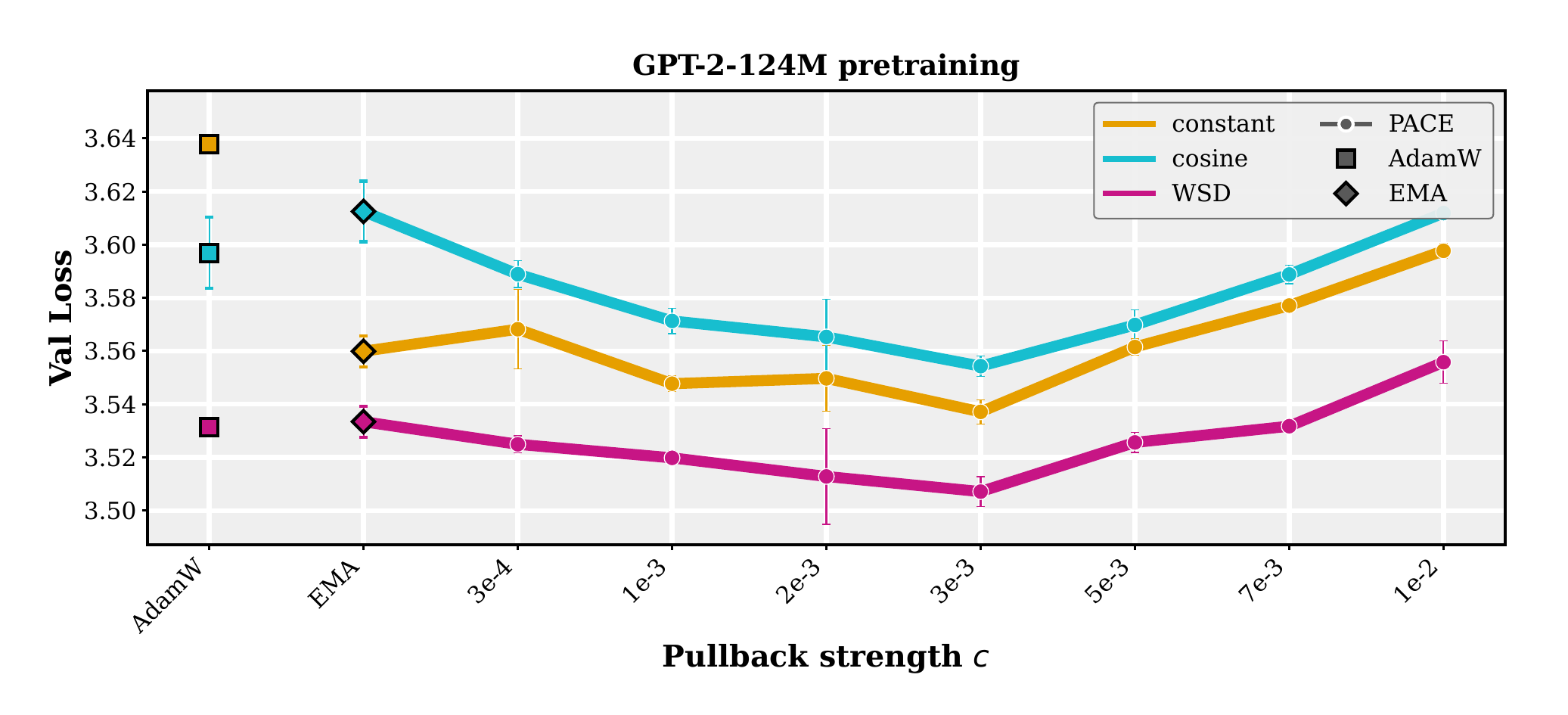}
  \caption{\textbf{Multi-seed pullback-strength sweep across learning-rate
  schedules} on GPT-2-124M pretraining. Mean returned-model validation
  cross-entropy with seed-to-seed standard deviation, comparing \pace\ to the
  AdamW and \ema\ baselines under each schedule. \pace\ improves on both
  baselines under every schedule, by more than the seed-to-seed variation.}
  \label{fig:appendix_pretrain_csweep}
\end{figure}

\begin{figure}[ht]
  \centering
  \includegraphics[width=\textwidth]{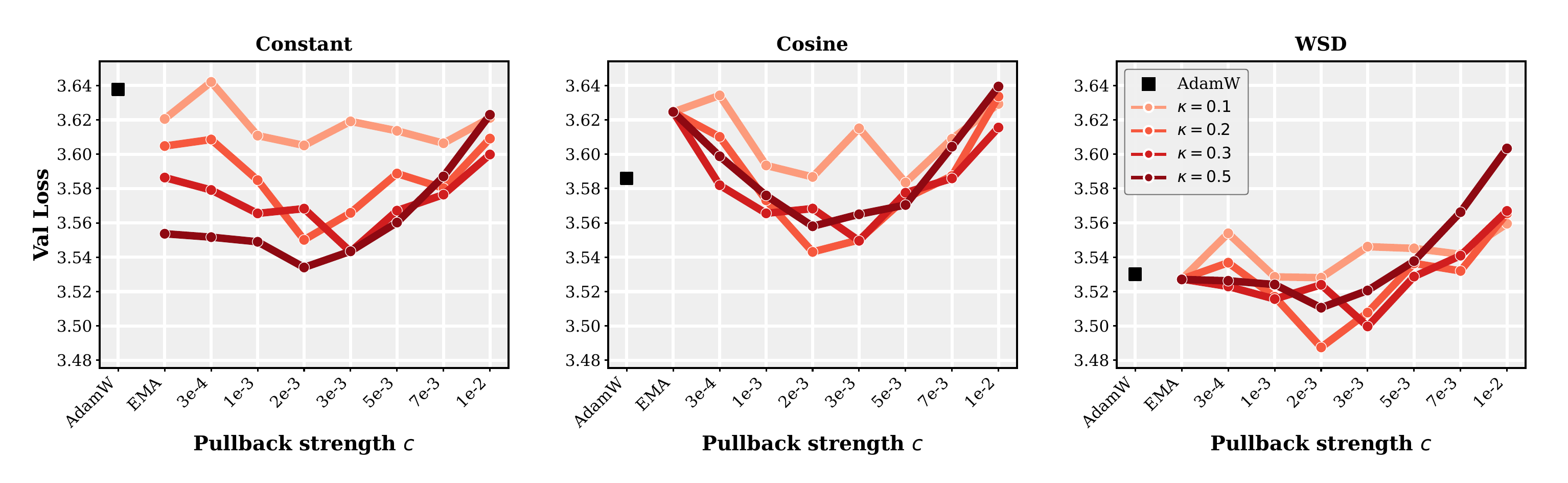}
  \caption{\textbf{Pullback-strength sweep resolved by EMA power} on
  GPT-2-124M pretraining (seed $42$), under a constant learning rate
  (\textbf{left}), cosine decay (\textbf{middle}), and WSD (\textbf{right}). A
  basin spanning $c\approx 2$--$3{\times}10^{-3}$ appears at every $\kappa$
  under every schedule.}
  \label{fig:appendix_pretrain_csweep_kappa}
\end{figure}

\section{Token-Budget Comparisons with WSD}
\label{sec:app-river}

In \Cref{fig:pretrain} (right) and \Cref{fig:bema_sf} (right) we compared
\pace\ at a constant learning rate to AdamW with linear decay to zero. In this
appendix we make that comparison systematic for pretraining and interpret both
regimes. The analysis is inspired by the token-budget comparisons used to
study Schedule-Free methods \citep{defazio2024schedule,song2025river}, and we
read it through the river-valley description of the loss landscape
\citep{wen2024warmup,song2025river}: at a constant learning rate, the iterate
oscillates along the steep ``hill'' directions of a valley while making steady
progress along its flat ``river'' floor, and the decay phase of WSD acts by
suppressing this oscillation and bringing the iterate down to the floor. This
perspective makes clear both the appeal of WSD and its cost: although the
decay phase produces the characteristic sharp drop in loss at the end of
training, until it begins the loss curve overstates the loss the model could
achieve, so that model quality is difficult to assess mid-training, and once
it begins, the remainder of the run is committed to a single token budget. An
optimizer whose returned iterate already lies on the river floor incurs
neither cost: its loss curve reflects post-decay performance at every step of
training, and training may be stopped or continued at any point.

The token-budget comparison tests exactly this property. Because a decay
branch primarily removes oscillation while making little further progress
along the river, the endpoint of a branch launched at a given budget estimates
the position of the river at that budget. Comparing the endpoints of branches
launched at several budgets to a candidate trajectory thus determines whether
that trajectory tracks the river: endpoints that coincide with the trajectory
indicate that it already lies on the floor, while endpoints that fall below it
reveal residual oscillation that a decay phase would have removed. We run
this comparison in both regimes---branches at $50\%/75\%/100\%$ of the
pretraining budget, and at three budgets and two learning rates in
fine-tuning---with all runs sharing data, batch size, and base learning rate.

\paragraph{Pretraining.}
In \Cref{fig:appendix_pretrain_river}, we overlay the AdamW decay branches on
the \pace\ constant-learning-rate trajectory. Every branch lands on the
trajectory: the endpoints ($3.721/3.597/3.530$) match \pace\ at the same budgets
($\approx 3.707/3.595/3.534$) to within $0.014$, and to within $0.004$ at the
full budget---closer than the seed-to-seed spread. At every budget, decaying the
learning rate only recovers the loss that \pace\ already reaches without a
schedule.

\begin{figure}[t]
  \centering
  \includegraphics[width=0.676\textwidth]{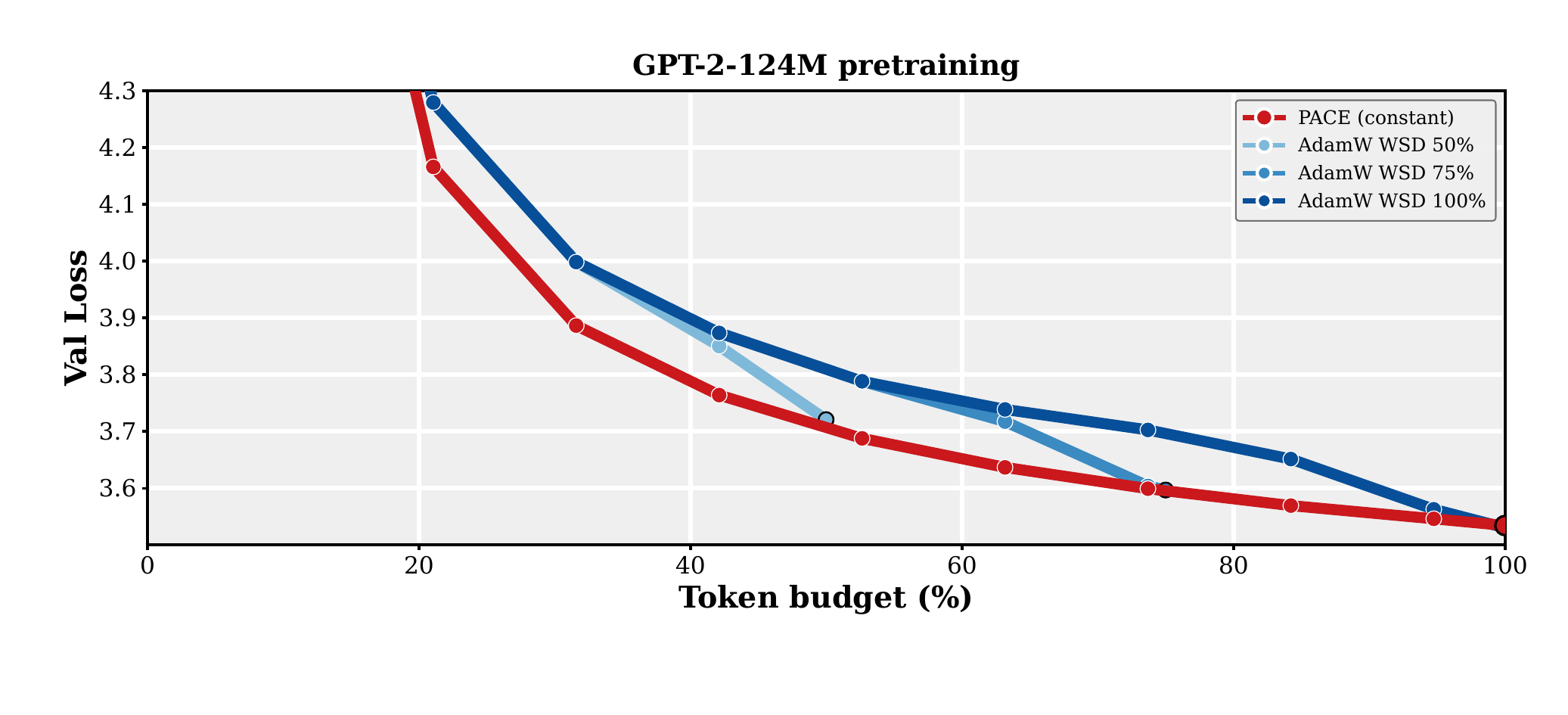}
  \caption{\textbf{Token-budget comparison on GPT-2-124M pretraining.}
  AdamW WSD decay branches at $50\%/75\%/100\%$ of the token budget overlaid on
  the \pace\ constant-learning-rate trajectory; each branch endpoint lands on
  the \pace\ trajectory, so the decay phase only recovers the loss that \pace\
  attains continuously without a schedule.}
  \label{fig:appendix_pretrain_river}
\end{figure}

\paragraph{Fine-tuning.}
The fine-tuning comparison in \Cref{fig:bema_sf} (right) yields a stronger
conclusion: every budget-matched branch ends above the \pace\ trajectory, with
endpoints of $1.277/1.233/1.205$ against \pace\ at $1.153$, and the gap widens
at larger learning rates. In this regime, no decay branch reaches the loss of
the constant-learning-rate \pace\ run at any budget.

\paragraph{Peeling off intermediate checkpoints to decay.}
Across every budget, learning rate, and regime that we tested, no decay
branch ends below the \pace\ trajectory, which is precisely the behavior this
comparison associates with a trajectory that already lies on the river floor.
The practical consequences are those identified above: the validation curve
of \pace\ reflects post-decay performance at every step, so that model
quality is visible throughout training, and the loss that WSD attains only at
the designated endpoint of a branch is available at every budget, without
fixing a training horizon in advance. In fine-tuning, we find the stronger
conclusion that the \pace\ trajectory lies below every budget-matched branch
endpoint, suggesting that the modified trajectory makes faster progress along
the river than the trajectory from which the branches descend.

\FloatBarrier
\section{Technical Preliminaries}

In this appendix, we review some technical preliminaries that will be useful for our analysis.  In \Cref{app:stoch_control}, we review some classical results from the theory of linear quadratic stochastic control, which we will use to derive our algorithm. Then in \Cref{app:ar_processes} we describe some classical results on autoregressive processes, which we will use to analyze the convergence of our algorithm in the quadratic optimization setting.

\subsection{Linear Quadratic Gaussian (LQG) Problems}\label{app:stoch_control}

In this section, we briefly review several classical results from the theory of linear quadratic stochastic control, which will be useful for our analysis.  Classical references for this material include \citet{kwakernaak1972linear,yong1999stochastic} among many others.

A \emph{Linear Quadratic Gaussian} (LQG) control problem is a stochastic control problem where the dynamics of the system are linear, the cost function is quadratic, and the noise is Gaussian.  More particularly, we suppose that $x_t^u \in \rr^d$ is a stochastic process satisfying
\begin{align}\label{eq:lqg-dynamics}
    d x_t^u = \left( \bA x_t^u + \bB u_t \right) d t + \bC W_t,
\end{align}
where $\bA, \bB, \bC$ are matrices of appropriate dimensions, $u_t$ is a control input (and thus progressively measurable with respect to the filtration generated by $W_t$), and $W_t$ is a standard Brownian motion in $\rr^d$.  For fixed time horizon $T$, we then seek to minimize a cost function of the form
\begin{align}\label{eq:lqg-cost}
    J(u) = \ee\left[ x_T^{u \top} \bQ_f x_T^u + \int_0^T \left( x_t^{u \top} \bQ x_t^u + u_t^\top \bR u_t \right) d t \right],
\end{align}
where $\bQ_f, \bQ, \bR$ are positive semidefinite matrices of appropriate dimensions.  The intuition for this cost is that we wish to minimize the expected deviation of the \emph{state} $x_t^u$ from the origin, while also ensuring that the control inputs $u_t$ are not too large; we pay both a \emph{final} cost for the end state $x_T^u$ being far from the origin as well as a \emph{running} cost for both intermediate states being far from the origin and control inputs being large.  This problem is classical and has been well-studied in the control theory literature.  We recall the fundamental result that the optimal control strategy is given by a linear state feedback controller, where the feedback gain is given by the solution to a matrix Riccati differential equation \citep{georgiou2013separation,kwakernaak1972linear}.
\begin{theorem}[\cite{yong1999stochastic} Theorem 6.1]\label{thm:lqg}
    Suppose that $x_t^u$ evolves according to \eqref{eq:lqg-dynamics} for some progressively measurable $u_t$.  Suppose further that $\bQ, \bQ_f, \bR$ are all symmetric and positive semidefinite and that $\bR$ is positive definite.  Let $\bP: [0, T] \to \rr^{d \times d}$ be the solution to the matrix Riccati differential equation
    \begin{align}\label{eq:riccati_general}
        - \bPdot_t = \bA^\top \bP_t + \bP_t \bA - \bP_t \bB \bR^{-1} \bB^\top \bP_t + \bQ, \quad \bP_T = \bQ_f,
    \end{align}
    where $\bPdot_t$ denotes the time derivative of $\bP_t$.  Then it holds that the optimal control $u_t$ that minimizes the cost function $J(u)$ defined in \eqref{eq:lqg-cost} is given by
    \begin{align}\label{eq:lqg-optimal-control}
        u_t = - \bR^{-1} \bB^\top \bP_t x_t^u
    \end{align}
    and the optimal cost is given by
    \begin{align}\label{eq:lqg-optimal-cost}
        J(u) = x_0^\top \bP_0 x_0 + \int_0^T \trace\left( \bP_t \bC \bC^\top \right) d t.
    \end{align}
\end{theorem}
While the Riccati equation provides the optimal controller in full generality, we will not require this in the sequel.  We thus state two corollaries, corresponding to special cases of the LQG problem, that will be sufficient for our purposes.  In the first, we will ignore any final state cost (so $\bQ_f = 0$), and assume that both $\bQ$ and $\bR$ are scalar multiples of the identity.
\begin{corollary}\label{cor:lqg-scalarcost}
    Suppose that $x_t^u$ evolves according to \eqref{eq:lqg-dynamics} for some progressively measurable $u_t$, and that $\bQ_f = 0$, $\bQ = q \cdot \eye$, and $\bR = r \cdot \eye$ for some $q, r > 0$.  Suppose further that $\bB = \eye$.  Then the optimal control $u_t$ that minimizes the cost function $J(u)$ defined in \eqref{eq:lqg-cost} is given by
    \begin{align}
        u_t = - r^{-1} \bP_t x_t^u, \quad \text{where} \quad - \bPdot_t = \bA^\top P_t + P_t \bA - r^{-1} P_t^2 + q \cdot \eye, \quad P_T = 0.
    \end{align}
\end{corollary}

In the second corollary, we will now ignore the \emph{running} state cost (so $\bQ = 0$), and assume that $\bQ_f$ and $\bR$ are scalar multiples of the identity.
\begin{corollary}\label{cor:lqg-finalcost}
    Suppose that $x_t^u$ evolves according to \eqref{eq:lqg-dynamics} for some progressively measurable $u_t$, and that $\bQ = 0$, $\bQ_f = q \cdot \eye$, and $\bR = r \cdot \eye$ for some $q, r > 0$.  Suppose further that $\bB = \eye$.  Then the optimal control $u_t$ that minimizes the cost function $J(u)$ defined in \eqref{eq:lqg-cost} is given by
    \begin{align}
        u_t = - r^{-1} \bP_t x_t^u, \quad \text{where} \quad  \bPdot_t = \bA^\top P_t + P_t \bA - r^{-1} P_t^2, \quad P_T = q \cdot \eye.
    \end{align}
\end{corollary}

Finally, we specialize to the fully scalar case, where $\bA$ is diagonal and prove here a closed-form solution to the Riccati equation.  While this calculation is standard, we include it here for completeness.
\begin{lemma}\label{lem:scalar-riccati-solution}
    Let $p: [0, T] \to \rr$ be the solution to the scalar Riccati equation
    \begin{align}
        - \dot{p}(t) = 2 a p(t) - r^{-1} p(t)^2 + q, \quad p(T) = q_f.
    \end{align}
    Let
    \begin{align}
        \beta_{\pm} = r ( a \pm \omega), \quad \text{where} \quad \omega = \sqrt{a^2 + r^{-1} q}.
    \end{align}
    Let
    \begin{align}
        \kappa(t) = \frac{q_f - \beta_+}{q_f - \beta_-} \cdot e^{- 2 \omega(T - t)}.
    \end{align}
    Then
    \begin{align}\label{eq:scalar-riccati-solution}
        p(t) = \frac{\beta_+ - \kappa(t) \beta_-}{1 - \kappa(t)}.
    \end{align}
\end{lemma}
\begin{proof}
    Note that by factoring we have
    \begin{align}
        \dot{p}(t) = r^{-1} \left( p(t) - \beta_+ \right)\left( p(t) - \beta_- \right).
    \end{align}
    Thus, when $\omega > 0$, we have by a standard separation of variables argument that
    \begin{align}
        \frac{p(t) - \beta_+}{p(t) - \beta_-} = \frac{q_f - \beta_+}{q_f - \beta_-} \cdot e^{- 2 \omega(T - t)}.
    \end{align}
    The result follows by rearranging.
\end{proof}

\subsection{Autoregressive Processes and the Yule-Walker Equations}\label{app:ar_processes}

We now review some classical results on autoregressive processes, which we will use to analyze the convergence of our algorithm in the quadratic optimization setting.  We refer to \citep{shumway2006time,brockwell2002introduction} for a more detailed review of this material.  We begin with a definition of autoregressive processes.
\begin{definition}\label{def:ar_process}
    We say a sequence of random variables $\{X_k\}_{k \geq 0}$ is an autoregressive process of order $p$ (AR($p$)) if it satisfies, for $k \geq p$,
    \begin{align}
        X_k = \sum_{i=1}^p \phi_i X_{k-i} + \epsilon_k,
    \end{align}
    for some $\phi_1, \dots, \phi_p \in \rr$ and some sequence of i.i.d. random variables $\{\epsilon_k\}_{k \geq 0}$ with mean zero and finite variance.  The \emph{characteristic polynomial} of the AR($p$) process is given by
    \begin{align}
        \phi(z) = 1 - \sum_{i=1}^p \phi_i z^i.
    \end{align}
    Moreover, we say that the process is \emph{causal} if it can be represented as
    \begin{align}
        X_k = \sum_{j = 0}^\infty \psi_j \epsilon_{k-j},
    \end{align}
    with the coefficients $\psi_j$ satisfying $\sum_{j=0}^\infty |\psi_j| < \infty$.
\end{definition}
Formally causality is necessary for us to ensure a limiting variance; fortunately there exists a simple characterization of causality in terms of the roots of the characteristic polynomial.
\begin{proposition}[Property 3.1 \cite{shumway2006time}]\label{prop:causality}
    An AR($p$) process is causal if and only if the roots of the characteristic polynomial $\phi(z)$ lie outside the unit circle in the complex plane.  In this case, $X_k$ has a well-defined limiting variance.
\end{proposition}
We now recall the Yule-Walker equations, which relate the coefficients $\phi_i$ of an AR($p$) process to the autocovariance function $\gamma(j) = \ee[X_k X_{k+j}]$.
\begin{theorem}[Section 3 \cite{shumway2006time}]\label{thm:yule-walker}
    Let $\{X_k\}_{k \geq 0}$ be a causal AR($p$) process with coefficients $\phi_1, \dots, \phi_p$ such that $\epsilon_k$ has variance $\sigma^2$. Then the autocovariance function $\gamma(j) = \ee[X_k X_{k+j}]$ satisfies the following equations for all $j \geq 0$:
    \begin{align}
        \gamma(j) = \sum_{i=1}^p \phi_i \gamma(j - i) + \sigma^2 \cdot \delta_{j, 0},
    \end{align}
    In particular, if $p = 2$, it holds that
    \begin{align}
        \gamma(1) &= \phi_1 \cdot \gamma(0) \\
        \gamma(2) &= \phi_1 \cdot \gamma(1) + \phi_2 \cdot \gamma(0) \\
        \gamma(0) &= \phi_1 \cdot \gamma(1) + \phi_2 \cdot \gamma(2) + \sigma^2.
    \end{align}
\end{theorem}
Finally, we recall the following result on the limiting variance of an AR($2$) process.
\begin{lemma}\label{lem:ar2-limiting-variance}
    Let $\{X_k\}_{k \geq 0}$ be a causal AR($2$) process with coefficients $\phi_1, \phi_2$ such that $\epsilon_k$ has variance $\sigma^2$ with fixed, deterministic $X_1, X_2$.  If $\phi(z) \neq 0$ for all complex $|z| \leq 1$, then the limiting variance of $X_k$ is given by
    \begin{align}
        \lim_{k \to \infty} \ee[X_k^2] = \gamma(0).
    \end{align}
\end{lemma}
\begin{proof}
    This follows immediately from Definition 3.7 and Property 3.1 of \citet{shumway2006time}. Indeed, let $X_k'$ denote the AR($2$) process with the same coefficients $\phi_1, \phi_2$ and noise $\epsilon_k$ but with zero initial conditions $X_1' = X_2' = 0$.  By causality, $X_k'$ has a well-defined limiting variance given by $\gamma(0)$.  Thus it suffices to show that $Y_k = X_k - X_k'$ has zero variance in the limit.  Observe that
    \begin{align}
        Y_k = \phi_1 Y_{k-1} + \phi_2 Y_{k-2},
    \end{align}
    so it remains to show that $Y_k \to 0$ as $k \to \infty$.  Writing out the transition matrix, we see that the eigenvalues thereof are the reciprocals of the roots of $\phi(z)$.  Because $\phi(z) \neq 0$ for all $|z| \leq 1$, the eigenvalues of the transition matrix are strictly less than $1$ in magnitude, so $Y_k \to 0$ as $k \to \infty$.
\end{proof}

\section{Mathematical Derivation of \pace}\label{app:stoch_control_derivation}\label{sec:app-stoch-control}

In this appendix we provide a complete formal derivation of our proposed algorithm.  Recall that we abstract the master optimization problem as the following stochastic control problem:
\begin{align}\label{eq:stoch_control_problem}
    \min_{u} J_T(u) = \ee\left[ \norm{\thetema[T] - \mustar}^2 + \lambda \cdot \int_0^T \norm{u_t}^2 d t \right], \quad  d \theta_t^u = \left[ \bA \left( \mustar - \theta_t^u \right) + u_t \right] d t + \bSigma^{1/2} d W_t.
\end{align}
In other words, we are running (a continuous time limit of) stochastic gradient descent on a quadratic objective $f(\theta) = \frac 12 \left(\theta - \mustar  \right)^\top \bA \left( \theta - \mustar\right)$ perturbed by some control input $u_t$, and we wish to choose $u_t$ so as to minimize the mean squared error of the returned iterate $\thetema[T]$ as an estimator of $\mustar$, plus a regularization term ensuring that the control inputs $u_t$ are not too large.  We provide a complete characterization of the optimal control strategy for this problem, and then show how to derive a simple approximation to this optimal control strategy that is both practical and empirically effective.

In \Cref{app:relaxed_objective} we first consider a relaxed objective, which leads to a simpler optimal control strategy that helps illustrate the main ideas of the analysis.  Then, in \Cref{app:original_objective}, we derive the optimal control strategy for the actual objective defined in \eqref{eq:stoch_control_problem}, subject to the assumption that $\mustar$ is known to the learner.  Because this is, of course, an unrealistic assumption, we then show how the principle of \emph{certainty equivalence} can be used to derive the optimal controller even when $\mustar$ is unknown, and how this optimal controller can be approximated to yield a simple and practical algorithm.  Finally, in \Cref{app:unknown_mustar}, we show how the optimal controller can be approximated even when $\mustar$ is unknown, and how this approximation leads to the form of the controller used in our proposed algorithm.

\subsection{Optimal Control for the Relaxed Objective}\label{app:relaxed_objective}

We begin with the simplest variation of the problem, where we assume $\mustar$ is known to the controller and we wish to minimize the relaxed objective
\begin{align}\label{eq:relaxed_objective}
    J_T'(u) = \ee\left[ T^{-1} \cdot\int_0^T \norm{\theta_t^u - \mustar}^2 d t + \lambda \cdot \int_0^T \norm{u_t}^2 d t \right].
\end{align}
While we actually want to minimize the $J_T(u)$ defined in \eqref{eq:stoch_control_problem}, which we will do in the sequel, but observe that (i) $J_T(u) \leq J_T'(u)$ for all $u$ by Jensen's inequality and (ii) minimizing $J_T'(u)$ leads to a more analytically tractable problem.

Recall that $\theta_t^u$ evolves according to the SDE \eqref{eq:sde},
\begin{align}
    d \theta_t^u = \left[ \bA \left( \mustar - \theta_t^u \right) + u_t \right] d t + \bSigma^{1/2} d W_t.
\end{align}
We claim the following result.
\begin{theorem}\label{thm:optimal_control_relaxed}
    Suppose that $\theta_t^u$ evolves according to \eqref{eq:sde} for some progressively measurable $u_t$.  Then the optimal control $u_t$ that minimizes the relaxed objective $J_T'(u)$ defined in \eqref{eq:relaxed_objective} is given by
    \begin{align}\label{eq:optimal_control_relaxed}
        u_t = \frac{\bP_t \left( \mustar - \theta_t^u \right)}{\lambda}, \quad \text{where} \quad \bPdot_t = \bA^\top P_t + P_t \bA + \lambda^{-1} P_t^2 - T^{-1} \eye, \quad \text{and} \quad P_T = 0.
    \end{align}
    In the special case that $\bA = \diag(\alpha_1, \dots, \alpha_d)$ is diagonal, we have that
    \begin{align}\label{eq:optimal_control_relaxed_diagonal}
        u_i(t) = \frac{1}{\lambda T} \cdot \frac{1 - e^{- 2 \sqrt{\alpha_i^2 + \nicefrac{1}{\lambda T}} (T -t)}}{\left( \sqrt{\alpha_i^2 + \nicefrac{1}{\lambda T}} + \alpha_i \right) + \left( \sqrt{\alpha_i^2 + \nicefrac{1}{\lambda T}} - \alpha_i \right) e^{- 2 \sqrt{\alpha_i^2 + \nicefrac{1}{\lambda T}} (T - t)}} \left( \mustar_i - \theta_i(t) \right).
    \end{align}
\end{theorem}
\begin{proof}
    The proof proceeds by first centering the dynamics and then applying the standard Riccati equation given in \Cref{thm:lqg} to solve for the optimal control.  Indeed, let
    \begin{align}\label{eq:centered-dynamics}
        x_t = \theta_t^u - \mustar, \quad \text{so} \quad d x_t = \left[ - \bA x_t + u_t \right] d t + \bSigma^{1/2} d W_t.
    \end{align}
    We can then rewrite $J_T'(u)$ as
    \begin{align}
        J_T'(u) = \ee\left[ \frac 1T \int_0^T \norm{x_t}^2 d t + \lambda \int_0^T \norm{u_t}^2 d t \right].
    \end{align}
    Now, writing
    \begin{align}
        \bQ = T^{-1} \cdot \eye \quad \text{and} \quad \bR = \lambda \cdot \eye,
    \end{align}
    we can again rewrite $J_T'(u)$ as
    \begin{align}
        J_T'(u) = \ee\left[ \int_0^T \left( x_t^\top \bQ x_t + u_t^\top \bR u_t \right) d t \right],
    \end{align}
    We now observe that combining \eqref{eq:centered-dynamics} with the above expression for $J_T'(u)$, we have reduced the problem of minimizing $J_T'(u)$ to a standard linear-quadratic-Gaussian (LQG) control problem as in \Cref{thm:lqg}.  Applying the Riccati equation from \Cref{thm:lqg} yields the desired result, in particular the case found in \Cref{cor:lqg-scalarcost}.

    For the special case, we use \Cref{lem:scalar-riccati-solution}.  Indeed, setting
    \begin{align}
        r = \lambda, \quad q = T^{-1}, \quad q_f = 0, \quad \text{and} \quad a = - \alpha_i,
    \end{align}
    that result tells us that we can solve the scalar Riccati equation to get
    \begin{align}
        p_i(t) = -\frac{\lambda \left( \sqrt{\alpha_i^2 + \nicefrac{1}{\lambda T}} - \alpha_i \right)\left( 1 - e^{- 2 \sqrt{\alpha_i^2 + \nicefrac{1}{\lambda T}} (T - t)} \right)}{1 + \frac{\sqrt{\alpha_i^2 + \nicefrac{1}{\lambda T}} - \alpha_i}{\sqrt{\alpha_i^2 + \nicefrac{1}{\lambda T}} + \alpha_i} e^{- 2 \sqrt{\alpha_i^2 + \nicefrac{1}{\lambda T}} (T - t)}}.
    \end{align}
    Thus,
    \begin{align}
        u_i(t) &= - \lambda^{-1} p_i(t) x_t^{(i)} \\
        &=\frac{ \left( \sqrt{\alpha_i^2 + \nicefrac{1}{\lambda T}} - \alpha_i \right)\left( 1 - e^{- 2 \sqrt{\alpha_i^2 + \nicefrac{1}{\lambda T}} (T - t)} \right)}{1 + \frac{\sqrt{\alpha_i^2 + \nicefrac{1}{\lambda T}} - \alpha_i}{\sqrt{\alpha_i^2 + \nicefrac{1}{\lambda T}} + \alpha_i} e^{- 2 \sqrt{\alpha_i^2 + \nicefrac{1}{\lambda T}} (T - t)}} (\mustar_i - \theta_t^{(i)}) \\
        &= \frac{1}{\lambda T} \cdot \frac{1 - e^{- 2 \sqrt{\alpha_i^2 + \nicefrac{1}{\lambda T}} (T - t)}}{\sqrt{\alpha_i^2 + \nicefrac{1}{\lambda T}} + \alpha_i + \left(\sqrt{\alpha_i^2 + \nicefrac{1}{\lambda T}} - \alpha_i\right)   e^{- 2 \sqrt{\alpha_i^2 + \nicefrac{1}{\lambda T}} (T - t)}}.
    \end{align}
    The result follows.
\end{proof}

We now construct an approximation to the above optimal control where we assume that $t \ll T$.  Indeed, we suppose that
\begin{align}
    T - t \gg \frac{1}{2 \sqrt{\alpha_i^2 + \nicefrac{1}{\lambda T}}} \geq \frac{1}{2 \alpha_i},
\end{align}
so that we can approximate $e^{- 2 \sqrt{\alpha_i^2 + \nicefrac{1}{\lambda T}} (T - t)} \approx 0$.  Then we have that
\begin{align}\label{eq:approx-optimal-control}
    u_i(t) \approx \frac{1}{\lambda T} \cdot \frac{1}{\sqrt{\alpha_i^2 + \nicefrac{1}{\lambda T}} + \alpha_i},
\end{align}
which is the form of the control used in our proposed algorithm.

\subsection{Optimal Algorithm for the Original Objective}\label{app:original_objective}

We now continue by considering the simplified setting where the controller is aware of the true minimizer $\mustar$, but where we wish to minimize the original objective $J_T(u)$ defined in \eqref{eq:stoch_control_problem} rather than the relaxed objective $J_T'(u)$ defined in \eqref{eq:relaxed_objective}.

\begin{theorem}\label{thm:original_objective}
    Suppose that $\theta_t^u$ evolves according to \eqref{eq:sde} for some progressively measurable $u_t$.  Then the optimal control $u_t$ that minimizes the original objective $J_T(u)$ defined in \eqref{eq:stoch_control_problem} is given by
    \begin{align}
        \ustar_t =  \lambda^{-1} \left( \bM_t (\mustar - \theta_t) + \bN_t \int_0^t (\mustar - \theta_s) d s \right),
    \end{align}
    where
    \begin{align}
        - \dot{\bM_t} &= - \bA^\top \bM_t - \bM_t \bA + \bN_t + \bN_t^\top - \lambda^{-1} \bM_t^2, \quad \bM_T = 0, \\
        - \dot{\bN_t} &= - \bA^\top \bN_t + \bS_t - \lambda^{-1} \bM_t \bN_t, \quad \bN_T = 0, \\
        - \dot{\bS_t} &= - \lambda^{-1} \bN_t^\top \bN_t, \quad \bS_T = T^{-2} \eye.
    \end{align}
    Moreover, in the special case that $\bA = \diag(\alpha_1, \dots, \alpha_d)$ is diagonal, we have that for $\alpha_i \neq 0$, 
    \begin{align}\label{eq:optimal_controller_diagonal_nonzero}
        \ustar_i(t) = \frac{1 - e^{- \alpha_i (T - t)}}{\alpha_i \lambda T^2 + \nicefrac{(T - t)}{\alpha_i} - \frac{2}{\alpha_i^2} \cdot \left( 1 - e^{- \alpha_i (T - t)} \right) + \frac{1 - e^{- 2 \alpha_i (T - t)}}{2 \alpha_i^2}} \left( \frac{1 - e^{- \alpha_i (T -t)}}{\alpha_i} \left( \mustar - \theta_t \right) + \int_0^t \left( \mustar - \theta_s \right) d s\right)
    \end{align}
    whereas for $\alpha_i = 0$, we have
    \begin{align}\label{eq:optimal_controller_diagonal_zero}
        \ustar_i(t) = \frac{(T - t)^2}{\lambda T^2 + \nicefrac{(T - t)^3}{3}} \cdot \left( \mustar - \theta_t \right) + \frac{(T - t)}{\lambda T^2 + \nicefrac{(T - t)^3}{3}} \cdot \int_0^t \left( \mustar - \theta_s \right) d s.
    \end{align}

\end{theorem}
\begin{proof}
    We introduce the accumulated centered process
    \begin{align}
        y_t^u = \int_0^t \left( \theta_s^u - \mustar \right) d s.
    \end{align}
    Then we observe that
    \begin{align}
        J_T(u) = \ee\left[ \frac 1{T^2} \norm{y_T}^2 + \lambda \int_0^T \norm{u_t}^2 d t \right], \quad \text{and} \quad d y_t^u = x_t d t.
    \end{align}
    This is not quite in the form of an LQG problem, but suppose that we augment the state to be $z_t^u = (x_t^u, y_t^u)$.  Then we have that
    \begin{align}
        d z_t^u = \left( \bF z_t + \bG u_t \right) d t + \bE d W_t, \quad \text{where} \quad \bF = \begin{bmatrix} -\bA & 0 \\ \eye & 0 \end{bmatrix}, \quad \bG = \begin{bmatrix} \eye \\ 0 \end{bmatrix}, \quad \bE = \begin{bmatrix} \bSigma^{1/2} \\ 0 \end{bmatrix}.
    \end{align}
    We then may take $\bQ$ to be zero and
    \begin{align}
        \bQ_f = \begin{bmatrix} 0 & 0 \\ 0 & T^{-2} \eye \end{bmatrix}, \quad \bR = \lambda \eye.
    \end{align}
    This then again becomes an LQG problem, and we can apply the Riccati equation from \Cref{thm:lqg} to solve for the optimal control.  Indeed, the Riccati equation in this case is given by
    \begin{align}
        - \dot{\bP_t} = \bF^\top \bP_t + \bP_t \bF - \lambda^{-1} \bP_t \bG \bG^\top \bP_t, \quad \text{and} \quad \bP_T = \bQ_f.
    \end{align}
    Writing
    \begin{align}\label{eq:riccati-decomposition}
        \bP_t = \begin{bmatrix}
            \bM_t & \bN_t \\ \bN_t^\top & \bS_t
        \end{bmatrix},
    \end{align}
    we get that
    \begin{align}
        \ustar_t = - \lambda^{-1} \left( \bM_t x_t + \bN_t y_t \right)
    \end{align}
    or, alternatively
    \begin{align}\label{eq:riccati_augmented_optimal_control}
        \ustar_t = - \lambda^{-1} \left( \bM_t (\theta_t - \mustar) + \bN_t \int_0^t \left( \theta_s - \mustar \right) d s \right),
    \end{align}
    where
    \begin{align}
        - \dot{\bM_t} &= - \bA^\top \bM_t - \bM_t \bA + \bN_t + \bN_t^\top - \lambda^{-1} \bM_t^2, \quad \bM_T = 0, \\
        - \dot{\bN_t} &= - \bA^\top \bN_t + \bS_t - \lambda^{-1} \bM_t \bN_t, \quad \bN_T = 0, \label{eq:decomposed_riccati} \\
        - \dot{\bS_t} &= - \lambda^{-1} \bN_t^\top \bN_t, \quad \bS_T = T^{-2} \eye
    \end{align}
    is derived by plugging in the decomposition in \eqref{eq:riccati-decomposition} into the Riccati equation in \eqref{eq:riccati_general} and rearranging.
    The first result follows.

    For the special case that $\bA$ is diagonal, the coordinates decouple and we derive the solution in \Cref{lem:original_objective_solution_scalar}.
\end{proof}

\begin{lemma}\label{lem:original_objective_solution_scalar}
    Suppose that $\bA = \diag(\alpha_1, \dots, \alpha_d)$ is diagonal.  Then the above system of ODEs for $\bM_t$, $\bN_t$, and $\bS_t$ decouples into $d$ independent systems of ODEs, one for each coordinate.  In particular, for each $i \in [d]$, we have that \eqref{eq:decomposed_riccati} decouples and can be solved to yield a closed-form solution for $\ustar_i(t)$, given in \eqref{eq:optimal_controller_diagonal_nonzero} and \eqref{eq:optimal_controller_diagonal_zero}.
\end{lemma}
\begin{proof}
    In the case where $\bA$ is diagonal, it is immediate that the system decouples into $d$ independent systems of ODEs, one for each coordinate because now every matrix is diagonal (note that $\bSigma$ may be taken to be diagonal without loss of generality by the rotational invariance of Brownian motion).  Thus, we can solve for each coordinate separately.  
    
    In particular, we know that the optimal controller is given by \eqref{eq:riccati_augmented_optimal_control}, i.e. for the $i^{th}$ coordinate, 
    \begin{align}
        \ustar_i(t) = \lambda^{-1} \left( M_t (\mustar_i - \theta_t^{(i)}) + N_t \int_0^t (\mustar_i - \theta_s^{(i)}) d s \right),
    \end{align}
    where $M_t$ and $N_t$ solve the system of ODEs from \eqref{eq:decomposed_riccati}  in the scalar setting, i.e.
    \begin{align}
        - \dot{M}_t = - 2 \alpha_i M_t + 2 N_t - \lambda^{-1} M_t^2, \quad M_T = 0, \\
        - \dot{N}_t = - \alpha_i N_t + S_t - \lambda^{-1} M_t N_t, \quad N_T = 0,\label{eq:scalar_riccati_odes} \\
        - \dot{S}_t = - \lambda^{-1} N_t^2, \quad S_T = T^{-2}.
    \end{align}
    Note that we are now in the scalar setting, so we are not bolding the variables; for notational simplicity, we omit the subscript $i$ in the remainder of this proof.  Now, let $\tau = T - t$ be the time to go and define the following functions:
    \begin{align}
        h(\tau) = \int_0^\tau e^{- \alpha s} d s = \begin{cases}
            \frac{1 - e^{- \alpha \tau}}{\alpha}, & \alpha \neq 0, \\
            \tau, & \alpha = 0,
        \end{cases}
    \end{align}
    and
    \begin{align}
        g(\tau) = \int_0^\tau h(s)^2 d s = \begin{cases}
            \frac{\tau}{\alpha^2} - \frac{2(1 - e^{- \alpha \tau})}{2\alpha^3} + \frac{1 - e^{- 2 \alpha \tau}}{2 \alpha^3}, & \alpha \neq 0, \\
            \frac{\tau^3}{3}, & \alpha = 0.
        \end{cases}
    \end{align}
    We claim that the solution to the above system of ODEs is given by
    \begin{align}\label{eq:scalar_riccati_solution_augmented}
        M_t = \frac{h(\tau)^2}{T^2 + \nicefrac{g(\tau)}{\lambda}}, \quad N_t = \frac{h(\tau)}{T^2 + \nicefrac{g(\tau)}{\lambda}}, \quad S_t = \frac{1}{T^2 + \nicefrac{g(\tau)}{\lambda}}.
    \end{align}
    Let $D(\tau) = T^2 + \nicefrac{g(\tau)}{\lambda}$.  We verify that the above is indeed a solution by plugging into the ODEs.  First, we have that
    \begin{align}
        h'(\tau) = e^{- \alpha \tau} = 1 - \alpha h(\tau), \quad g'(\tau) = h(\tau)^2, \quad \text{and} \quad D'(\tau) = \lambda^{-1} h(\tau)^2.
    \end{align}
    Observing that $\nicefrac{d}{d \tau} = - \nicefrac{d}{d t}$, can plug in \eqref{eq:scalar_riccati_solution_augmented} into \eqref{eq:scalar_riccati_odes} to observe that this is indeed the solution.

    Thus, plugging back in, we see that the optimal control is given by
    \begin{align}
        \ustar(t) &= \frac{h(T - t)}{\lambda T^2 + g(T - t)} \cdot \left( h(T -t) (\mustar - \theta_t) + \int_0^t \left( \mustar - \theta_s \right) ds \right).
    \end{align}
    In particular, for $\alpha = 0$, we have that
    \begin{align}
        \ustar(t) = \frac{(T - t)^2}{\lambda T^2 + \nicefrac{(T - t)^3}{3}} \cdot \left( \mustar - \theta_t \right) + \frac{(T - t)}{\lambda T^2 + \nicefrac{(T - t)^3}{3}} \cdot \int_0^t \left( \mustar - \theta_s \right) d s.
    \end{align}
    For $\alpha \neq 0$, we have that
    \begin{align}
        \ustar(t) = \frac{1 - e^{- \alpha (T - t)}}{\alpha \lambda T + \nicefrac{(T - t)}{\alpha} - \frac{2}{ \alpha^2} \cdot \left( 1 - e^{- \alpha (T - t)} \right) + \frac{1 - e^{- 2 \alpha_i (T - t)}}{2 \alpha_i^2}} \left( \frac{1 - e^{- \alpha (T -t)}}{\alpha} \left( \mustar - \theta_t \right) + \int_0^t \left( \mustar - \theta_s \right) d s\right).
    \end{align}
    The result follows.
\end{proof}

\subsection{Generalizing to the Case Where $\mustar$ is Unknown}\label{app:unknown_mustar}

The above results are a promising start but all rely on the assumption that $\mustar$ is known to the controller, which is certainly not the case in practice.  Indeed, the entire goal of optimization is to \emph{find} $\mustar$; thus, if the learner had access to $\mustar$ from the start there would be no need to run any optimization algorithm at all.  In this section we show how to modify the above optimal controllers to the case where $\mustar$ is unknown, and in particular we show that the same form of controller is still optimal, but with $\mustar$ replaced by an estimate of $\mustar$.

The key factor in generalizing to the case where $\mustar$ is unknown is the \emph{separation principle} from control theory, which is also known as \emph{certainty equivalence} \citep{georgiou2013separation,van1981certainty,kwakernaak1972linear}.  This principle essentially states that for an LQG control problem where the state is not directly observed, the optimal control can be obtained by first estimating the state as the conditional expectation of the state given the observations and then plugging this estimate into the optimal control for the fully observed case.  More formally, we have the following result.
\begin{theorem}\label{thm:separation_principle}
    Consider an LQG problem as in \eqref{eq:lqg-dynamics} and \eqref{eq:lqg-cost} where the state $x_t^u$ is not directly observed, but rather we have access to an observation process $y_t = \phi(x_t^u, \mustar)$ for some unknown $\mustar$ drawn from a known prior distribution with $\phi$ invertible in the first argument.  Suppose that the corresponding full-information LQG problem where the state is directly observed has an optimal control given by $u_t = \bK_t x_t$ for some matrix $\bK_t$.  Then the optimal control for the partially observed problem is given by $u_t = \bK_t \hat{x}_t$, where $\hat{x}_t = \ee\left[ \phi^{-1}(y_t, \mustar) \mid y_s, s \leq t \right]$ is the conditional expectation of the state given the observations.  
\end{theorem}
\begin{proof}
    This is a classical result in control theory, and we refer the reader to \citet{georgiou2013separation,van1981certainty,kwakernaak1972linear} for a complete proof.  In order to translate those works to our setting, we may apply \citet[Theorem 14]{georgiou2013separation} to the original control problem augmented with the state $\mu_t$ such that $\mu_0 = \mustar$ and $d \mu_t = 0$ for all $t$; the observation process $\theta_t = \mustar + x_t$ is then a linear function of the state, and the optimal control for the fully observed problem is given by $u_t = \bK_t x_t$ where $\bK_t$ is the solution to the Riccati equation as in \Cref{thm:lqg}.  Applying \citet[Theorem 14]{georgiou2013separation} then yields the desired result.
\end{proof}
We thus have the following two corollaries, for the relaxed and original objectives respectively.
\begin{corollary}
    Let $J_T'(u)$ denote the relaxed objective in \eqref{eq:relaxed_objective} and suppose that $\theta_t^u$ evolves according to \eqref{eq:sde} for some progressively measurable $u_t$.  Let $\muhat_t = \ee\left[ \mustar \mid \theta_s^u, s \leq t \right]$ be the conditional expectation of $\mustar$ given the trajectory of $\theta_t^u$ up to time $t$.  Then the optimal control $u_t$ that minimizes $J_T'(u)$ is given by
    \begin{align}
        \ustar(t) = \frac{\bP_t \left( \muhat_t - \theta_t^u \right)}{\lambda}
    \end{align}
    where $\bP_t$ is the solution to the same Riccati equation as in \eqref{eq:optimal_control_relaxed}.  In the special case that $\bA = \diag(\alpha_1, \dots, \alpha_d)$ is diagonal, we have that
    \begin{align}
        \ustar_i(t) = \frac{1}{\lambda T} \cdot \frac{1 - e^{- 2 \sqrt{\alpha_i^2 + \nicefrac{1}{\lambda T}} (T -t)}}{\left( \sqrt{\alpha_i^2 + \nicefrac{1}{\lambda T}} + \alpha_i \right) + \left( \sqrt{\alpha_i^2 + \nicefrac{1}{\lambda T}} - \alpha_i \right) e^{- 2 \sqrt{\alpha_i^2 + \nicefrac{1}{\lambda T}} (T - t)}} \left( \muhat_i(t) - \theta_i(t) \right)
    \end{align}
    for each $i \in [d]$.
\end{corollary}
\begin{proof}
    This follows immediately by combining \Cref{thm:optimal_control_relaxed} with \Cref{thm:separation_principle}.
\end{proof}
For the original objective, we have the following result.
\begin{corollary}\label{cor:optimal_control_unknown_mustar}
    Let $J_T(u)$ denote the original objective in \eqref{eq:stoch_control_problem} and suppose that $\theta_t^u$ evolves according to \eqref{eq:sde} for some progressively measurable $u_t$.  Let $\muhat_t = \ee\left[ \mustar \mid \theta_s^u, s \leq t \right]$ be the conditional expectation of $\mustar$ given the trajectory of $\theta_t^u$ up to time $t$.  Then the optimal control $u_t$ that minimizes $J_T(u)$ is given by
    \begin{align}
        \ustar_t =  \lambda^{-1} \left( \bM_t (\muhat_t - \theta_t) + \bN_t \left(t \cdot \muhat_t  - \int_0^t \theta_s d s \right) \right),
    \end{align}
    where $\bM_t$, $\bN_t$, and $\bS_t$ are the solutions to the same system of ODEs as in \Cref{thm:original_objective}.  In the special case that $\bA = \diag(\alpha_1, \dots, \alpha_d)$ is diagonal, we have that for $\alpha_i \neq 0$,
    \begin{align}
        \ustar_i(t) = \frac{1 - e^{- \alpha_i (T - t)}}{\alpha_i \lambda T^2 + \nicefrac{(T - t)}{\alpha_i} - \frac{3}{2 \alpha_i^2} \cdot \left( 1 - e^{- \alpha_i (T - t)} \right)} \left( \frac{1 - e^{- \alpha_i (T -t)}}{\alpha_i} \left( \muhat_i(t) - \theta_t^{(i)} \right) + \left(t \cdot \muhat_i(t) - \int_0^t  \theta_s^{(i)}  d s\right)\right)
    \end{align}
    whereas for $\alpha_i = 0$, we have
    \begin{align}
        \ustar_i(t) = \frac{(T - t)^2}{\lambda T^2 + \nicefrac{(T - t)^3}{3}} \cdot \left( \muhat_i(t) - \theta_t^{(i)} \right) + \frac{(T - t)}{\lambda T^2 + \nicefrac{(T - t)^3}{3}} \cdot \left(t \cdot \muhat_i(t) - \int_0^t \theta_s^{(i)}  d s \right).
    \end{align}
\end{corollary}
\begin{proof}
    This follows immediately by combining \Cref{thm:original_objective} with \Cref{thm:separation_principle}.
\end{proof}

Thus, at least in theory, we have a complete answer to the question of how to control optimization algorithms in order to minimize the mean squared error of the returned iterate average as an estimator of $\mustar$.  Unfortunately, the above optimal controllers are not implementable in practice because they require computing $\muhat_t = \ee\left[ \mustar \mid \theta_s^u, s \leq t \right]$, even in the simplified diagonal setting where the coordinates decouple.  In the following subsection, we show how to derive a practical algorithm by approximating $\muhat_t$ with a simple estimator of $\mustar$ that can be computed from the trajectory of $\theta_t^u$ up to time $t$.

\subsection{Deriving a Practical Algorithm through Approximation}\label{app:practical_algorithm}

As remarked above, the optimal controllers derived in the previous subsections are not implementable in practice.  For the general setting, solving the Riccati equations is already highly nontrivial in high dimensions and certainly beyond reach at the scale of language models.  Even in the diagonal setting, the optimal controllers require computing $\muhat_t = \ee\left[ \mustar \mid \theta_s^u, s \leq t \right]$, which is not computable in closed form and would require solving a high-dimensional integral at every time step.  Moreover, the updates themselves are quite complicated.  We thus propose three simplifications:
\begin{enumerate}
    \item We will consider only the diagonal setting, where the coordinates decouple and we can solve for the optimal controller in closed form.
    \item We will approximate $\muhat_t = \frac 1t \int_0^t \theta_s^u d s$, which is a natural estimator of $\mustar$ given the trajectory of $\theta_t^u$ up to time $t$.
    \item We will approximate the optimal controller by its form when $t \ll T$, which is the regime we are in for most of the optimization trajectory.
\end{enumerate}
The first approximation allows for the closed-form solution to the optimal controller, which we derived above.  The second approximation is a natural one, as $\muhat_t$ becomes simple to estimate and is naturally aligned with the core motivation of this work: using iterate averaging to get better estimates of $\mustar$.  The third approximation is also natural and can be used to greatly simplify the form of the optimal controller.  In particular we will make the following approximations:
\begin{align}
    T -t \approx T \quad \text{and} \quad e^{- \alpha_i(T - t)} \approx 0,
\end{align}
where the second of these relies on the assumption that $T - t \gg \frac{1}{2 \sqrt{\alpha_i^2 + \nicefrac{1}{\lambda T}}}$.  Critically, by replacing $\muhat_t$ with the iterate average, the second term in \Cref{cor:optimal_control_unknown_mustar} drops out, which is a significant simplification and makes the resulting controller much more practical to implement.

With these approximations, we have that the optimal controller in \Cref{cor:optimal_control_unknown_mustar} is approximated by
\begin{align}
    \ustar_i(t) = \frac{1}{T \left(1 + \lambda \alpha_i^2 T\right)} \left( \muhat_i(t) - \theta_i(t) \right),
\end{align}
which is a simple and practical controller that can be empirically implemented.  In particular, we can write the above in vector form as
\begin{align}\label{eq:approx_optimal_control_unknown_mustar}
    \ustar(t) = \frac{1}{T} \cdot \left( \eye + \lambda \bA^2 T \right)^{-1} \left(\muhat(t) - \theta(t) \right).
\end{align}
This is the form of the controller used in our proposed algorithm, albeit with a discrete time approximation and additional simplifications.

\section{Proofs of Convergence Results}\label{app:convex_proofs}

In this section we provide rigorous proofs of the convergence results stated in \Cref{sec:theory}.  We begin with a proof of \Cref{thm:convex_convergence} in \Cref{app:convex_convergence}, before proceeding to prove our two propositions about the improvement over SGD in the quadratic setting in \Cref{app:improved_convergence_quadratic,app:arbitrarily_large_improvement}.

\subsection{Proof of Theorem \ref{thm:convex_convergence}}\label{app:convex_convergence}

We now prove that the update strategy given in \eqref{eq:stylized_update} converges to the minimum of $F$ when $F$ is convex.  In particular, we instantiate the update with a diagonal $\bC$ and a fixed EMA parameter $\beta$.  More precisely, we suppose
\begin{align}\label{eq:convex_update_params}
    \bC = \diag(c_1, \dots, c_d), \quad \text{where} \quad 0 \leq c_i \leq 1, \quad \text{and} \quad \beta \in (0, 1).
\end{align}
We denote by $\cmin = \min_{1 \leq i \leq d} c_i$ the smallest diagonal entry of $\bC$.
We will further make the following assumptions about the objective $F$ and the stochastic gradients $g_k$:
\begin{assumption}\label{ass:convexity}
    The objective $F: \rr^d \to \rr$ is convex and $G$-Lipschitz with minimum $\mustar$.
\end{assumption}
\begin{assumption}\label{ass:stochastic_gradients}
    At any iteration $k$, the gradient estimate $g_k \in \rr^d$ satisfies $\ee_{k-1}\left[ g_k \right] = \nabla F(\theta_k)$ and $\ee_{k-1}\left[ \norm{g_k - \nabla F(\theta_k)}^2 \right] \leq \sigma^2$ for some $\sigma > 0$, where $\ee_{k}\left[ \cdot \right]$ denotes expectation conditional on the filtration generated by $\theta_j, g_j$ for $j \leq k$.  We denote $V^2 = G^2 + \sigma^2$.
\end{assumption}
Recall that given parameter $\theta_k$, EMA parameter $\thetema[k]$, and stochastic gradient $g_k$, for learning rate $\eta_k$ we update
\begin{align}\label{eq:update}
    \theta_{k+1} = \theta_k - \eta_k \cdot g_k + \bC\left( \thetema[k] - \theta_k \right) \quad \text{and} \quad \thetema[k+1] = (1 - \beta) \cdot \thetema[k] + \beta \cdot \theta_{k+1}.
\end{align}

We now prove a more general result for arbitrary non-increasing learning rate schedules from which \Cref{thm:convex_convergence} will follow.
\begin{theorem}\label{thm:convex_convergence_general}
    Suppose that \Cref{ass:convexity,ass:stochastic_gradients} hold and that $\theta_k$ is updated as in \eqref{eq:update} for some non-increasing sequence $\{\eta_k\}_{k=0}^\infty$.  Furthermore, suppose that $\bC$ and $\beta$ are as in \eqref{eq:convex_update_params}.  Let
    \begin{align}
        \thetabar_T = \frac{\sum_{k  =1}^T \eta_k \cdot \theta_k}{\sum_{k  =1}^T \eta_k}
    \end{align}
    denote the average.  Let
    \begin{align}\label{eq:dc_kc}
        D_{\bC}^2 =  \cdot \left( \theta_1 - \mustar \right)^\top \left(  \eye + \nicefrac{1-\beta}{\beta} \cdot \bC\right) \left( \theta_1 - \mustar \right) \quad \text{and} \quad K_\bC = \max_{1 \leq i \leq d} \frac{c_i}{\beta + (1 - \beta) c_i}.
    \end{align}
    If $V$ is as in \Cref{ass:stochastic_gradients}, then
    \begin{align}
        \ee\left[ F(\thetabar_T) - F(\mustar) \right] \leq \frac{D_{\bC}^2 + \left( V^2 + \frac{2 G(1 - \beta) K_{\bC} V}{\beta + (1 - \beta) c_{\min}} \right) \cdot \sum_{k=1}^T \eta_k^2}{2 \cdot \sum_{k = 1}^T \eta_k}.
    \end{align}
\end{theorem}
\begin{proof}
    We introduce the notation
    \begin{align}
        r_k = \theta_k - \thetema[k].
    \end{align}
    From \eqref{eq:update}, we have
    \begin{align}
        \theta_{k+1} &= \thetema[k] + \left( \eye - \bC \right) r_k - \eta_k \cdot g_k.
    \end{align}
    On the other hand, by the definition of the EMA update, we have
    \begin{align}
        \thetema[k+1] = \thetema[k] + \beta \left( (\eye - \bC) r_k - \eta_k \cdot g_k \right).
    \end{align}
    Combining the preceding two displays, we have
    \begin{align}\label{eq:r_update}
        r_{k+1} = (1 - \beta)\left( \left( \eye - \bC \right) r_k - \eta_k \cdot g_k \right).
    \end{align}
    We now introduce the notation\footnote{This will never be confused with the $\bB$ in the definition of LQG control problems in \eqref{eq:lqg-dynamics}, as this notation is only used in this section.}
    \begin{align}\label{eq:B_def}
        \bB = \beta \cdot \eye + (1 - \beta) \cdot \bC.
    \end{align}
    Now, letting
    \begin{align}\label{eq:x_def}
        x_k = \thetema[k] + \beta \bB^{-1} (\eye - \bC) r_k,
    \end{align}
    we claim that
    \begin{align}\label{eq:x_update}
        x_{k+1} = x_k - \beta \eta_k \cdot \bB^{-1} g_k.
    \end{align}
    To see this, note that
    \begin{align}
        x_{k+1} &= \thetema[k+1] + \beta \bB^{-1}(\eye - \bC) r_{k+1} \\
        &= \thetema[k] + \left( \beta \cdot \eye + (1 - \beta) \beta \bB^{-1} (\eye - \bC) \right)\left( (\eye - \bC) r_k - \eta_k \cdot g_k \right) \\
        &= \thetema[k] + \beta \cdot \bB^{-1} (\eye - \bC) r_k - \beta \eta_k \cdot \bB^{-1} g_k \\
        &= x_k - \beta \eta_k \cdot \bB^{-1} g_k.
     \end{align} 

     We now claim that
     \begin{align}\label{eq:gk_bound}
        \ee_k\left[ g_k^\top \bB^{-1} g_k \right] \leq \beta^{-1} \cdot V^2.
     \end{align}
     To see this, note first that because $\bC$ is diagonal, it holds that $\bB^{-1}$ is diagonal with entries
     \begin{align}
        0 \leq b_i = \frac{1}{\beta + (1 - \beta) c_i} \leq \beta^{-1}.
     \end{align}
     Thus it holds that
     \begin{align}
        g_k^\top \bB^{-1} g_k \leq \beta^{-1} \norm{g_k}^2.
     \end{align}
    Thus, \eqref{eq:gk_bound} then follows by \Cref{ass:stochastic_gradients}.

     We now use the convexity of $F$ to bound control  the instantatneous suboptimality of $\theta_k$ by a function of $x_k$.  First, note that
     \begin{align}
        \eye - \beta \bB^{-1} (\eye - \bC) = \left( \beta \eye + (1 - \beta) \bC \right)^{-1} \bC = \bB^{-1} \bC.
     \end{align}
     Thus,
     \begin{align}\label{eq:beta_B_bound}
        \normop{\eye - \beta \bB^{-1} (\eye - \bC)} = \normop{\bB^{-1} \bC} = K_{\bC}.
     \end{align}
     Now, using \eqref{eq:x_def}, we see that
     \begin{align}
        \inprod{\nabla F(\theta_k)}{x_k - \mustar} &= \inprod{\nabla F(\theta_k)}{\theta_k - \mustar} - \inprod{\nabla F(\theta_k)}{\left( \eye - \beta \bB^{-1}(\eye - \bC) r_k \right)} \\
        &\geq F(\theta_k) - F(\mustar) - \inprod{\nabla F(\theta_k)}{\left( \eye - \beta \bB^{-1}(\eye - \bC) r_k \right)} \\
        &\geq F(\theta_k) - F(\mustar) - G K_{\bC} \norm{r_k},
     \end{align}
     where the first inequality follows from convexity of $F$ and the second inequality follows from Cauchy-Schwarz coupled with the $G$-Lipschitzness of $F$ and \eqref{eq:beta_B_bound}.

     Now, expanding the square, we see that
     \begin{align}
        \ee_k\left[ \beta^{-1} \left( x_{k+1} - \mustar \right)^\top  \bB \left( x_{k+1} - \mustar \right) \right] &= \ee_k\left[ \beta^{-1} \left( x_k  - \mustar - \beta \eta_k \cdot \bB^{-1} g_k \right)^\top \bB \left( x_k  - \mustar - \beta \eta_k \cdot \bB^{-1} g_k \right) \right] \\
        &= \beta^{-1} \left( x_k - \mustar \right)^\top \bB \left( x_k - \mustar \right) \\
        &\quad - 2 \eta_k \cdot \inprod{\nabla F(\theta_k)}{x_k - \mustar} + \beta \eta_k^2 \cdot \ee_k\left[ g_k^\top \bB^{-1} g_k \right] \\
        &\leq \beta^{-1} \left( x_k - \mustar \right)^\top \bB \left( x_k - \mustar \right) \\
        &\quad - 2 \eta_k \cdot \inprod{\nabla F(\theta_k)}{x_k - \mustar} + \beta \eta_k^2 \cdot V^2.
     \end{align}
     Plugging this into the preceding display and rearranging, we see that
     \begin{align}
        2 \eta_k \cdot \ee\left[ F(\theta_k) - F(\mustar) \right] &\leq \ee\left[  \beta^{-1} \left( x_k - \mustar \right)^\top \bB \left( x_k - \mustar \right) \right] - \ee\left[ \beta^{-1} \left( x_{k+1} - \mustar \right)^\top  \bB \left( x_{k+1} - \mustar \right) \right] \\
        &\quad + 2 \eta_k G K_{\bC} \cdot \ee\left[ \norm{r_k} \right] +  V^2 \cdot \eta_k^2.
     \end{align}
     Now, summing over $k$ and telescoping, we see that
     \begin{align}
        2 \cdot \sum_{k  =1}^T \eta_k \cdot \ee\left[ F(\theta_k) - F(\mustar) \right] &\leq \beta^{-1} \cdot \ee\left[ \left( x_1 - \mustar \right)^\top \bB \left( x_1 - \mustar \right) \right] \\
        &\quad + V^2 \cdot \sum_{k  =1}^T \eta_k^2 + 2 G K_{\bC} \cdot \sum_{k=1}^T \eta_k \cdot \ee\left[ \norm{r_k} \right].
     \end{align}
     Because $\theta_1 = \thetema[1]$ we know that $r_1 = 0$ and thus $x_1 = \theta_1$.  Thus the first term on the right hand side can be bounded by $D_{\bC}^2$ by the definition of $D_{\bC}^2$ in \eqref{eq:dc_kc}.  For the left hand side, observe that by convexity of $F$ we have that
     \begin{align}
        2 \cdot \sum_{k  =1}^T \eta_k \cdot \ee\left[ F(\theta_k) - F(\mustar) \right] \geq \left( 2 \cdot \sum_{k  =1}^T \eta_k \right) \cdot \ee\left[ F(\thetabar_T) - F(\mustar) \right].
     \end{align}
     Thus, all that remains is to bound the final term on the right hand side.  Applying the convexity of the norm to \eqref{eq:r_update}, we have that
     \begin{align}
        \norm{r_{k+1}} \leq (1 - \beta) \cdot \normop{\eye - \bC} \cdot \norm{r_k} + (1 - \beta) \cdot \eta_k \cdot \norm{g_k}.
     \end{align} 
     By the construction of $\bC$, we know that $\normop{\eye - \bC} = 1 - \cmin$.  Thus,
     \begin{align}
        \ee\left[ \norm{r_{k+1}} \right] \leq (1 - \beta)(1 - \cmin) \cdot \ee\left[ \norm{r_k} \right] + (1 - \beta) \cdot \eta_k \cdot V.
     \end{align}
     By induction, it thus holds that
     \begin{align}
        \ee\left\{ \norm{r_k} \right\} \leq (1 - \beta) V \sum_{j=0}^{k-1} \left( (1 - \beta)(1 - \cmin) \right)^j \cdot \eta_{k-1-j}.
     \end{align}
     Summing, and using the fact that $\eta_k$ is non-increasing, we then have that
     \begin{align}
        \sum_{k  = 1}^T \eta_k \cdot \ee\left[ \norm{r_k} \right] &\leq (1 - \beta) V \cdot \sum_{k = 1}^{T - 1}  \sum_{j = 0}^{k-2} \eta_{k-1-j}^2\left( (1 - \beta)(1 - \cmin) \right)^j \\
        &\leq \frac{(1 - \beta) V}{1 - (1 - \beta)(1 - \cmin)} \cdot \sum_{k=1}^T \eta_k^2.
     \end{align}
     The result concludes by plugging this into the preceding display and rearranging.
\end{proof}
We now prove \Cref{thm:convex_convergence} as a corollary of \Cref{thm:convex_convergence_general}.

We do this by stating a more explicit convergence guarantee for the special case of a fixed learning rate.

\begin{corollary}\label{cor:convex_convergence_fixed_lr}
    Suppose that \Cref{ass:convexity,ass:stochastic_gradients} hold and $\theta_k$ is updated as in \eqref{eq:update} for $\eta_k = T^{-\nicefrac 12} \cdot \eta$ for some fixed $\eta$.  Then,
    \begin{align}
        \thetabar_T = \frac{1}{T} \sum_{k=1}^T \theta_k
    \end{align}
    and
    \begin{align}
        \ee\left[ F(\thetabar_T) - F(\mustar) \right] \leq \frac{D_{\bC}^2 + \eta^2 \left( V^2 + \frac{2 G (1 - \beta) K_{\bC} V}{\beta + (1 - \beta)\cmin} \right)}{2 \eta \sqrt{T}}.
    \end{align}
    In particular, for the optimal choice of $\eta$,
    \begin{align}
        \ee\left[ F(\thetabar_T) - F(\mustar) \right] \leq \frac{D_{\bC} \sqrt{V^2 + \frac{2 G (1 - \beta) K_{\bC} V}{\beta + (1 - \beta)\cmin} }}{\sqrt{T}} \leq  \frac{\sqrt{2 - \beta}}{\beta} \cdot \frac{\norm{\theta_1 - \mustar} V}{\sqrt{T}}
    \end{align}
\end{corollary}
\begin{proof}
    The first statement on $\thetabar_T$ is immediate from the definition.  The final inequaltiy for the tuned bound follows from the fact that
    \begin{align}
        D_{\bC} \leq \frac{\norm{\theta_1 - \mustar}}{\sqrt{\beta}}, \quad K_{\bC} \leq 1, \quad \text{and} \quad \beta + (1 - \beta) \cmin \geq \beta.
    \end{align}
    The main bound for arbitrary fixed $\eta$ then follows by plugging the fixed learning rate schedule into \Cref{thm:convex_convergence_general}.
\end{proof}

Note that \Cref{thm:convex_convergence} follows immediately from \Cref{cor:convex_convergence_fixed_lr}.

\subsection{Proof of Proposition \ref{prop:improved_convergence_quadratic}}\label{app:improved_convergence_quadratic}

We now consider the special case of a diagonal quadratic objective, where we can obtain a more explicit convergence guarantee.  Suppose that
\begin{align}\label{eq:quadratic_objective}
    F(\theta) = \frac 12 (\theta - \mustar)^\top \bA (\theta - \mustar), \quad \text{where} \quad \bA = \diag(\alpha_1, \dots, \alpha_d) \quad \text{with} \quad \alpha_i > 0 .
\end{align}
Suppose that at each iteration we have access to a stochastic gradient of the form
\begin{align}\label{eq:quadratic_stochastic_gradient}
    g_k = \bA (\theta_k - \mustar) + \xi_k, \quad \text{where} \quad \ee\left[ \xi_k \right] = 0 \quad \text{and} \quad \Cov\left( \xi_k \right) = \diag\left( \sigma_1^2, \dots, \sigma_d^2 \right).
\end{align}

We will consider the update in \eqref{eq:update} for an arbitrary diagonal $\bC$ with nonnegative entries and a fixed EMA parameter $\beta$ and learning rate $\eta$.  We make the following assumption on the learning rate, which is standard in the analysis of stochastic optimization algorithms for quadratic objectives \citep{zhang2019algorithmic} and ensures that the algorithm is stable.
\begin{assumption}\label{ass:stable_learning}
    We assume that the learning rate $\eta$ satisfies $0 < \eta \alpha_i < 1$ for all $1 \leq i \leq d$.
\end{assumption}
Note that in the special case the $\bC = 0$, we recover the dynamics of an EMA on SGD precisely.
\begin{proposition}\label{prop:improved_convergence_quadratic_tight}
    Suppose that $F$ is as in \eqref{eq:quadratic_objective} and gradient estimates $g_k$ are as in \eqref{eq:quadratic_stochastic_gradient}.  Let $\thetahat_T$ denote the update at time $T$ given by \eqref{eq:update} for fixed $\eta$ satisfying \Cref{ass:stable_learning}.  Let $\thetasgd$ evolve such that $\thetasgd_{k+1} = \thetasgd_k - \eta \cdot g_k$ and let $\thetema[k+1] = (1 - \beta) \thetema[k] + \beta \theta_k$.  Then there exists some $\bC$ such that
    \begin{align}
        \lim_{T \to \infty} \ee\left[ \norm{\thetahat_T - \mustar}^2 \right] < \lim_{T \to \infty} \ee\left[ \norm{\thetema[T] - \mustar}^2 \right] < \lim_{T \to \infty} \ee\left[ \norm{\thetasgd_T - \mustar}^2 \right].
    \end{align}
    Thus, in particular,
    \begin{align}
        \lim_{T \to \infty} \ee\left[ F(\thetahat_T) \right] < \lim_{T \to \infty} \ee\left[ F(\thetema[T]) \right] < \lim_{T \to \infty} \ee\left[ F(\thetasgd_T) \right].
    \end{align}
\end{proposition}
\begin{proof}
    Note first that we may set $\mustar = 0$ without loss of generality; if $\mustar \neq 0$, we can simply shift the objective by $\mustar$ to arrive at the same setting and the same convergence guarantee. Second, note that the ultimate statement follows immediately from the penultimate statement by the definition of $F$ in \eqref{eq:quadratic_objective}.  Third, note that due to the diagonal structure of the objective and $\bC$, the dynamics of each coordinate are decoupled, and thus it suffices to prove the result for the one-dimensional case.  Thus, we will now assume that $d = 1$ and write $\alpha = \alpha_1$, $\sigma^2 = \sigma_1^2$, and $c = c_1$.  For the sake of notational convenience, let $a = \eta \cdot \alpha$.  The update then becomes
    \begin{align}\label{eq:decoupled_update}
        \theta_{k+1} = (1 - a - c) \theta_k + c \thetahat_k 0 - \eta \xi_k \quad \text{and} \quad \thetahat_{k+1} = (1 - \beta) \thetahat_k + \beta \theta_{k+1}.
    \end{align}
    We may write $\theta_t$ as a function of $\thetahat$ through the second equation in \eqref{eq:decoupled_update} and, plugging this into the first equation and rearranging, we get
    \begin{align}\label{eq:ar_update}
        \thetahat_{k+1} = \left( 2 - a - \beta - c(1 - \beta) \right) \thetahat_k - \left( 1 - \beta \right)(1 - a - c) \thetahat_{k-1} - \beta \eta \xi_k.
    \end{align}
    We may rewrite \eqref{eq:ar_update} as an explicity AR(2) process with noise by observing that
    \begin{align}
        \begin{bmatrix}
            \thetahat_{k+1} \\
            \thetahat_k
        \end{bmatrix} = \begin{bmatrix}
            2 - a - \beta - c(1 - \beta) & -(1 - \beta)(1 - a - c) \\
            1 & 0
        \end{bmatrix} \begin{bmatrix}
            \thetahat_k \\
            \thetahat_{k-1}
        \end{bmatrix} + \begin{bmatrix}
            -\beta \eta \xi_k \\
            0
        \end{bmatrix}.
    \end{align}
    An elementary shows that the eigenvalues of the transition matrix at $c=0$ are $1 - \beta$ and $1 - a$, which are both in the unit disk for sufficiently small $c$ by continuity.  Letting $\gamma_j(c) = \ee\left[ \thetahat_k \thetahat_{k - j} \right]$ denote the autocovariance of the process, the Yule-Walker equations for this AR(2) process \Cref{thm:yule-walker} are given by
    \begin{align}
        \gamma_1(c) &= \left( 2 - a - \beta - c(1 - \beta) \right) \gamma_0(c) - \left( 1 - \beta \right)(1 - a - c) \gamma_1(c) \\
        \gamma_2(c) &= \left( 2 - a - \beta - c(1 - \beta) \right) \gamma_1(c) - \left( 1 - \beta \right)(1 - a - c) \gamma_0(c) \\
        \gamma_0(c) &= \left( 2 - a - \beta - c(1 - \beta) \right) \gamma_1(c) - \left( 1 - \beta \right)(1 - a - c) \gamma_2(c) + \beta^2 \eta^2 \sigma^2.
    \end{align}
    Solving for $\gamma_0(c)$ yields
    \begin{align}\label{eq:gamma_0}
        \gamma_0(c) = \frac{\beta \eta^2 \sigma^2 \left( 2 - a - \beta + a \beta - c (1 - \beta) \right)}{a \left( a + \beta - a \beta + c (1 - \beta) \right)\left( 4 - 2a - 2 \beta + a \beta - 2 c (1 - \beta) \right)}.
    \end{align}
    By \Cref{lem:ar2-limiting-variance}, it holds that
    \begin{align}
        \lim_{T \to \infty} \ee\left[ \norm{\thetahat_T - \mustar}^2 \right] = \gamma_0(c).
    \end{align}
    Taking the derivative of $\gamma_0(c)$ with respect to $c$ at $c = 0$, we see that
    \begin{align}
        \gamma_0'(0) = - \frac{2 \beta(1 - \beta) \eta^2 \sigma^2 \left( (1 - \beta)^2 (1 - a)^2 + (1 - \beta)(2 - a) + (1 - a) \right)}{a(2-a)^2(2-\beta)^2(a + \beta - a \beta)^2} < 0.
    \end{align}
    Thus it holds that for sufficiently small $c > 0$,
    \begin{align}
        \lim_{T \to \infty} \ee\left[ \norm{\thetahat_T - \mustar}^2 \right] = \gamma_0(c) < \gamma_0(0) = \lim_{T \to \infty} \ee\left[ \norm{\thetema[T] - \mustar}^2 \right].
    \end{align}
    On the other hand, because
    \begin{align}
        \thetasgd_{k+1} = (1 - a) \thetasgd_k - \eta \xi_k,
    \end{align}
    it holds that
    \begin{align}
        \lim_{T \to \infty} \ee\left[ \norm{\thetasgd_T - \mustar}^2 \right] = \frac{\eta^2 \sigma^2}{1 - (1 - a)^2} = \frac{\eta^2 \sigma^2}{a(2 - a)} > \gamma_0(0),
    \end{align}
    where the inequality follows immediately from an elementary computation.
\end{proof}

\subsection{Proof of Proposition \ref{prop:arbitrarily_large_improvement}}\label{app:arbitrarily_large_improvement}

We now further specialize applying \eqref{eq:stylized_update} to the problem outlined \eqref{eq:quadratic_objective} and \eqref{eq:quadratic_stochastic_gradient} in the setting where $\bC = \eye$.  We state a more explicit convergence guarantee in this setting, which will allow us to show that the improvement over EMA on SGD can be arbitrarily large.
\begin{proposition}
    Suppose that $F$ is as in \eqref{eq:quadratic_objective} and gradient estimates $g_k$ are as in \eqref{eq:quadratic_stochastic_gradient}.  Let $\bA = \diag(\alpha_1, \dots, \alpha_d)$ let $\eta, \beta$ such that
    \begin{align}
        \beta < \eta \cdot \alpha_i < \nicefrac 18
    \end{align}
    for all $1 \leq i \leq d$.  Let $\thetahat_T$ denote the update at time $T$ given by \eqref{eq:update} for $\bC = \eye$ and fixed $\eta$.  Let $\thetasgd$ evolve such that $\thetasgd_{k+1} = \thetasgd_k - \eta \cdot g_k$ and let $\thetema[k+1] = (1 - \beta) \thetema[k] + \beta \theta_k$.  Then
    \begin{align}
        \lim_{T \to \infty}\ee\left[ \norm{\thetahat_T - \mustar}^2 \right] \leq 8 \lim_{T \to \infty} \eta \cdot \ee\left[ \norm{\bA^{\nicefrac 12}\left( \thetasgd_T - \mustar \right)}^2 \right].
    \end{align}
\end{proposition}
\begin{proof}
    As in the proof of \Cref{prop:improved_convergence_quadratic_tight}, we may reduce to a one-dimensional setting, writing $\alpha = \alpha_1$, $\sigma^2 = \sigma_1^2$, and $a = \eta \alpha$.  Observe that for $c = 1$, the coefficients are
    \begin{align}
        \phi_1 = 1 - a \quad \text{and} \quad \phi_2 = a (1 - \beta).
    \end{align}
    Thus, in this setting, by \eqref{eq:gamma_0}, it holds that
    \begin{align}
        \lim_{T \to \infty}\ee\left[ \norm{\thetahat_T - \mustar}^2 \right] = \gamma_0(1) = \frac{\beta \eta^2 \sigma^2 \left( 1 - a + a \beta \right)}{a (1 + a - a \beta)(2 - 2a + a \beta)},
    \end{align}
    whereas
    \begin{align}
        \lim_{T \to \infty} \ee\left[ \norm{\thetema[T] - \mustar}^2 \right] = \gamma_0(0) = \frac{\beta \eta^2 \sigma^2\left( 2 - a - \beta + a\beta \right)}{a(a + \beta - a \beta)(4 - 2a - 2\beta + a \beta)}.
    \end{align}
    Thus,
    \begin{align}
        \frac{\lim_{T \to \infty}\ee\left[ \norm{\thetahat_T - \mustar}^2 \right]}{\lim_{T \to \infty} \ee\left[ \norm{\thetema[T] - \mustar}^2 \right]} &= \frac{(2 - a)(2 - \beta)(a + \beta - a \beta)(1 - a + a \beta)}{(1  + a - a \beta)(2 - 2 a + a \beta)(2 - a - \beta + a \beta)} \\
        &\leq \frac{8 a}{(1  + a - a \beta)(2 - 2 a + a \beta)(2 - a - \beta + a \beta)} \\
        &\leq 8 a,
    \end{align}
    where in the first inequality we used the fact that
    \begin{align}
        2 -a \leq 2, \quad 2 - \beta \leq 2, \quad a + \beta - a \beta \leq a + \beta \leq 2 a, \quad \text{and} \quad 1 - a + a \beta \leq 1.
    \end{align}
    In the second inequality, we used the fact that
    \begin{align}
        1 + a - a \beta \geq 1 \quad \text{and} \quad 2 - 2a + a \beta \geq 2 - 2 a \geq 1.
    \end{align}
    The result then follows.
\end{proof}

The proof of \Cref{prop:arbitrarily_large_improvement} then follows immediately from the preceding proposition.

\end{document}